\documentclass[11pt]{article}
\usepackage[truedimen,margin=25truemm]{geometry}
 \pdfoutput=1 
\usepackage{microtype}
\usepackage{graphicx}
\usepackage{subfigure}
\usepackage{booktabs} 
\usepackage{amsmath}
\usepackage{amsfonts}
\usepackage{amssymb}
\usepackage{mathtools}
\usepackage{amsthm}
\usepackage{mathtools}
\usepackage{epstopdf}
\usepackage{multirow}
\usepackage{algorithm}
\usepackage{nicefrac}       
\usepackage{color}
\usepackage{graphics}
\usepackage{epstopdf}

\usepackage{hyperref}

\DeclareMathAlphabet\ten{OMS}{cmsy}{b}{n} 
\def\mat#1{\mbox{\bf #1}}
\newcommand\norm[1]{\left\lVert#1\right\rVert}

\newtheorem{lemma}{Lemma}
\newtheorem{theorem}{Theorem}
\newtheorem{corollary}{Corollary}
\newcommand{\spect}[1]{\mathcal{#1}}
\newcommand{\uu}[1]{\mat{U}_{#1}\mat{U}_{#1}^\top}
\newcommand{\spa}[1]{#1_{\Omega}}

\def\R{\mathbb{R}}

\def\P{\mathcal{P}}

\def\grad{\mathop{\rm grad}\nolimits}

\def\bu{\mathbf u}

\def\bQ{\mathbf Q}

\def\bU{{\mathbf U}}

\def\bX{{\mathbf X}}

\def\bZ{{\mathbf Z}}

\def\maxop{\mathop{\rm max}\limits} 
\def\minop{\mathop{\rm min}\limits}

\newcommand{\trace}{{\rm trace}}
\newcommand{\range}{{\rm range}}

\def\min{\mathop{\rm min}\nolimits}
\def\max{\mathop{\rm max}\nolimits}

\newcommand{\OG}[1]{{\mathcal{O}({#1})}}

\def\bI{{\mathbf I}}

\def\bQ{{\mathbf Q}}

\def\bU{{\mathbf U}}

\def\bX{{\mathbf X}}

\def\bZ{{\mathbf Z}}
\newcommand{\hess}{\mathrm{Hess}}

\newcommand{\changeBM}[1]{\textcolor{black}{#1}}

\newtheorem{remark}{Remark}

\title{A dual framework for low-rank tensor completion }

\author{Madhav Nimishakavi$^\ddag$ \\{\tt madhav@iisc.ac.in} \and Pratik Jawanpuria$^\dag$ \\{\tt pratik.jawanpuria@microsoft.com} \and Bamdev Mishra$^\dag$ \\{\tt bamdevm@microsoft.com}}

\date{$^\ddag$Indian Institute of Science \\
$^\dag$Microsoft}
\begin{document}

\maketitle

\begin{abstract}

One of the popular approaches for low-rank tensor completion is to use the \textit{latent trace norm} regularization. However, most existing works in this direction learn a sparse combination of tensors. In this work, we fill this gap by proposing a variant of the latent trace norm that helps in learning a non-sparse combination of tensors. We develop a dual framework for solving the low-rank tensor completion problem. We first show a novel characterization of the dual solution space with an interesting factorization of the optimal solution. Overall, the optimal solution is shown to lie on a Cartesian  product of Riemannian manifolds. Furthermore, we exploit the versatile Riemannian optimization framework for proposing computationally efficient trust region algorithm. The experiments illustrate the efficacy of the proposed algorithm on several real-world datasets across applications.

\end{abstract}



\section{Introduction}\label{sec:intro}
Tensors are multidimensional or $K$-way arrays, which provide a natural way to represent multi-modal data \cite{cichocki2017a,cichocki2017b}. Low-rank tensor completion problem, in particular, aims to recover a low-rank tensor from partially observed tensor \cite{acar2011a}. 
This problem has numerous applications in image/video inpainting \cite{liu2013a,Kressner2014a}, link-prediction \cite{ermis2015a}, and recommendation systems \cite{symeonidis2008a}, to name a few. 

In this work, we focus on trace norm regularized low-rank tensor completion problem of the form 	
\begin{equation}\label{eqn1}
    \min\limits_{\ten{W} \in \mathbb{R}^{n_1\times n_2 \times \ldots \times n_K}}  \|\spa{\ten{W}}-\spa{\ten{Y}}\|_F^2 + \frac{1}{\lambda} R(\ten{W}),
\end{equation}
where $\ten{Y}_\Omega \in \mathbb{R}^{n_1\times \ldots \times n_K}$ is a partially observed $K$-	mode tensor, whose entries are only known for a subset of indices $\Omega$. $(\spa{\ten{W}})_{(i_1,\ldots,i_K)} = \ten{W}_{(i_1,\ldots,i_K)}$, if $(i_1, \ldots, i_K) \in \Omega$ and $0$ otherwise, $\norm{\cdot}_F$ is the Frobenius norm , $R(\cdot)$ is a low-rank promoting regularizer, and $\lambda > 0$ is the regularization parameter.  

Similar to the matrix completion problem, the trace norm regularization has been used to enforce the low-rank constraint for the tensor completion problem. 
The works \cite{tomioka2010a,tomioka2013a} discuss the \emph{overlapped} and \emph{latent} trace norm regularizations for tensors. 
In particular, \cite{tomioka2013a,wimalawarne2014a} show that the latent trace norm has {certain} better tensor reconstruction bounds. 
The latent trace norm regularization learns the tensor as a \emph{sparse} combination of different tensors. 
In our work, we empirically motivate the need for learning non-sparse combination of tensors and propose a variant of the latent trace norm that learns a \emph{non-sparse} combination of tensors. We show a novel characterization of the solution space that allows for a compact storage of the tensor, thereby allowing to develop scalable optimization formulations. 
Concretely, we make the following contributions in this paper.
\begin{itemize}
\item We propose a novel trace norm regularizer for low-rank tensor completion problem, which learns a tensor as a non-sparse combination of tensors. In contrast, the more popular latent trace norm regularizer \cite{tomioka2010a,tomioka2013a,wimalawarne2014a} learns a highly sparse combination of tensors. Non-sparse combination helps in capturing information along all the modes. 
 \item We propose a dual framework for analyzing the problem formulation. This provides {interesting} insights into the solution space of the tensor completion problem, e.g., how the solutions along different modes are related,  allowing a compact representation of the tensor. 
\item Exploiting the characterization of the solution space, we develop {a} fixed-rank formulation. Our optimization problem is on Riemannian spectrahedron manifolds and we propose computationally efficient trust-region algorithm for our formulation.
\end{itemize}

Numerical comparisons on real-world datasets for different applications such as video and hyperspectral-image completion, link prediction, and movie recommendation show that the proposed algorithm outperforms state-of-the-art latent trace norm regularized algorithms.

The organization of the paper is as follows. Related works are discussed in Section \ref{sec:related}. In Section \ref{sec:formulation}, we study a particular trace norm based tensor completion formulation. Theorem \ref{dual_theorem} shows the characterization of the solution space. Building on this, we show two particular fixed-rank formulations. Both of these problems have the structure of optimization on Riemannian manifolds. The optimization related ingredients and their computational cost are subsequently discussed in Section \ref{sec:opt}. In Section \ref{sec:exp}, numerical comparisons on real data sets for three different applications: video and hyperspectral-image completion, link prediction, and movie recommendation, show that our proposed algorithms outperform various state-of-the-art latent trace norm regularized algorithms.

The Matlab codes are available at \url{https://pratikjawanpuria.com/}.



\section{Related work}
\label{sec:related}

\begin{figure}[b]
\centering
\begin{tabular}{cccc}
\begin{minipage}[b]{0.25\hsize}
\centering
\includegraphics[scale=0.30]{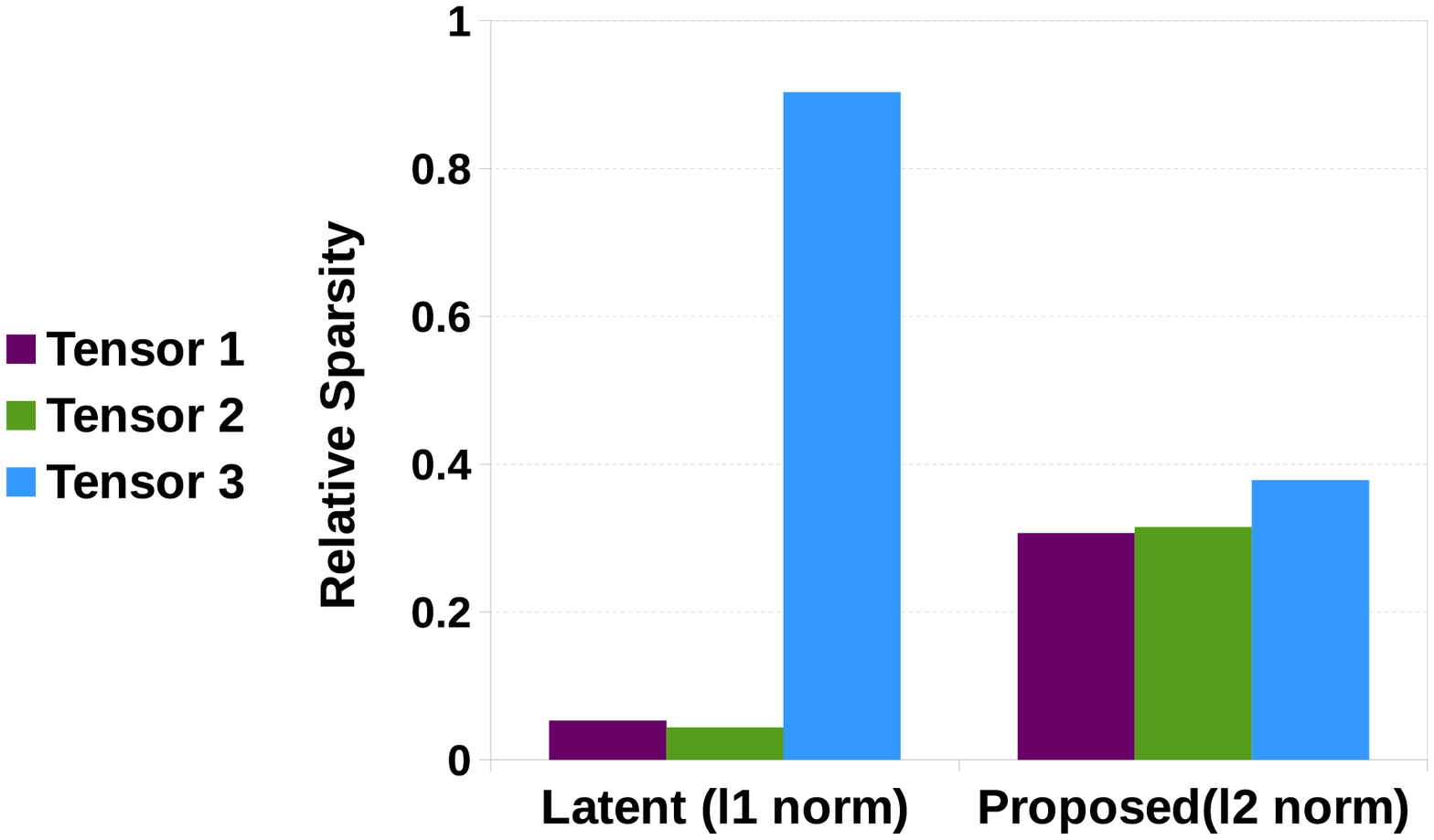}\\
{\small \qquad \qquad (a) Ribeira}
\end{minipage}
\begin{minipage}[b]{0.25\hsize}
\centering
\includegraphics[scale=0.30]{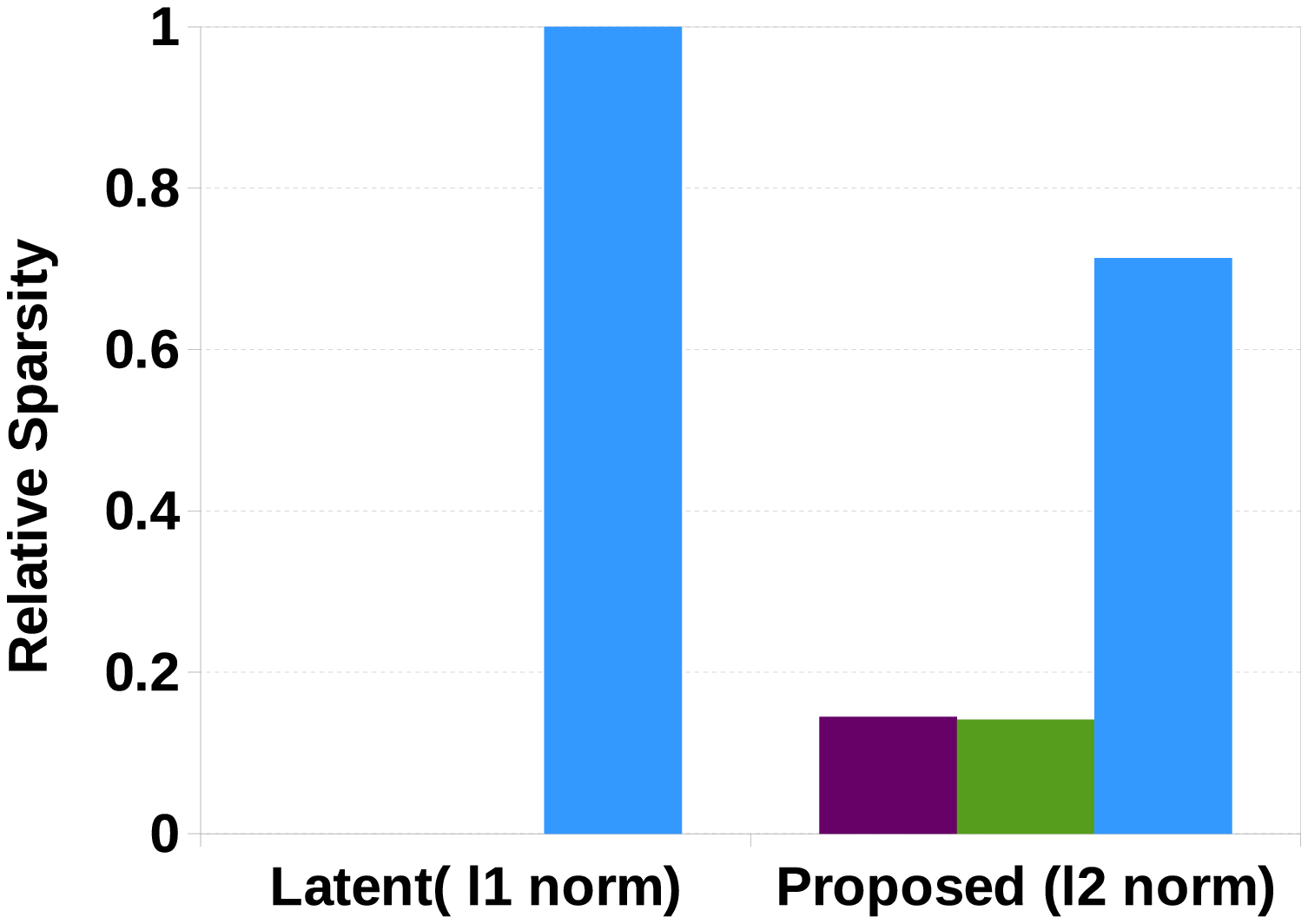}\\
{\small (b) Baboon }
\end{minipage}
\begin{minipage}[b]{0.25\hsize}
\centering
\includegraphics[scale=0.22]{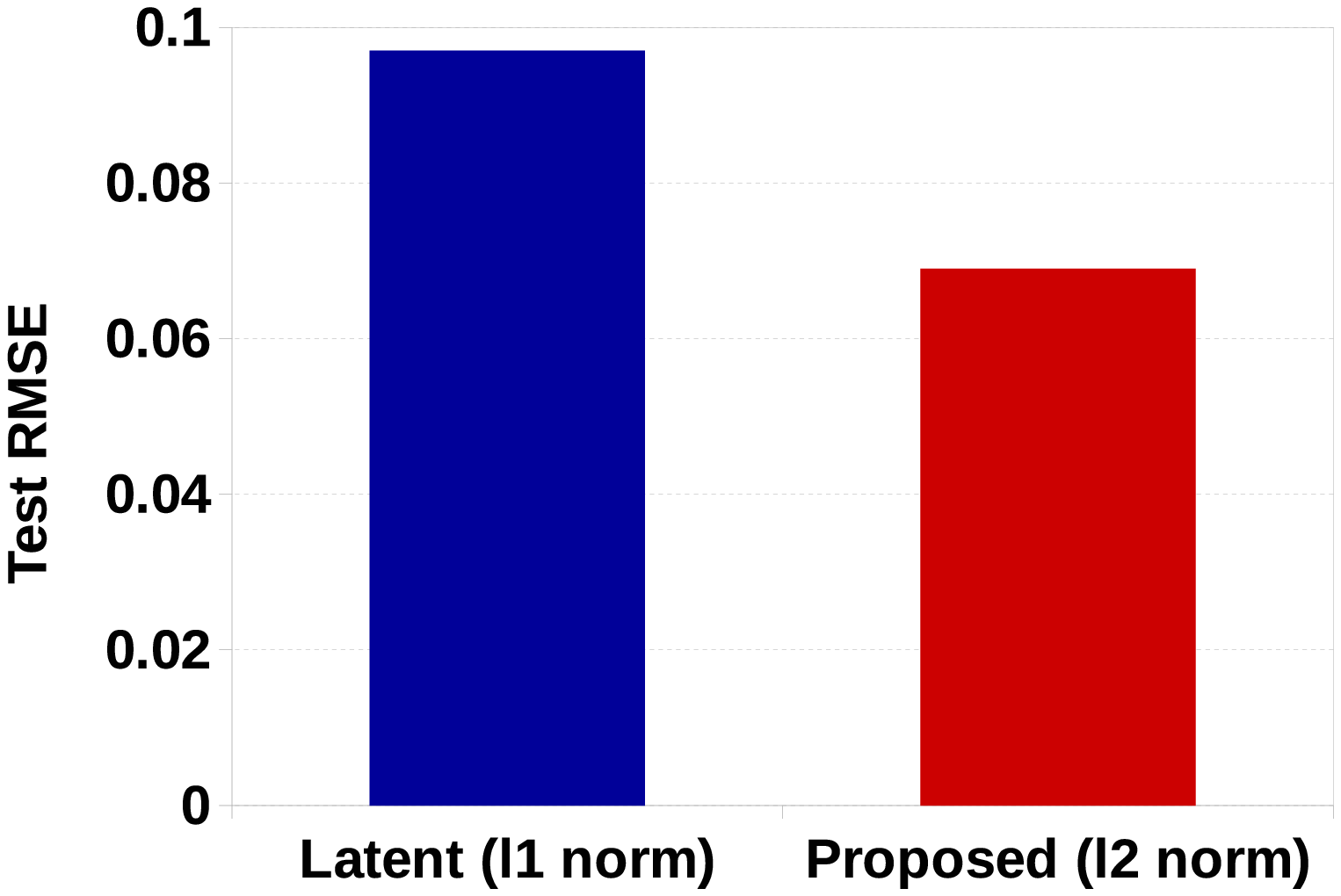}\\
{\small (c) Ribeira }
\end{minipage}
\begin{minipage}[b]{0.25\hsize}
\centering
\includegraphics[scale=0.22]{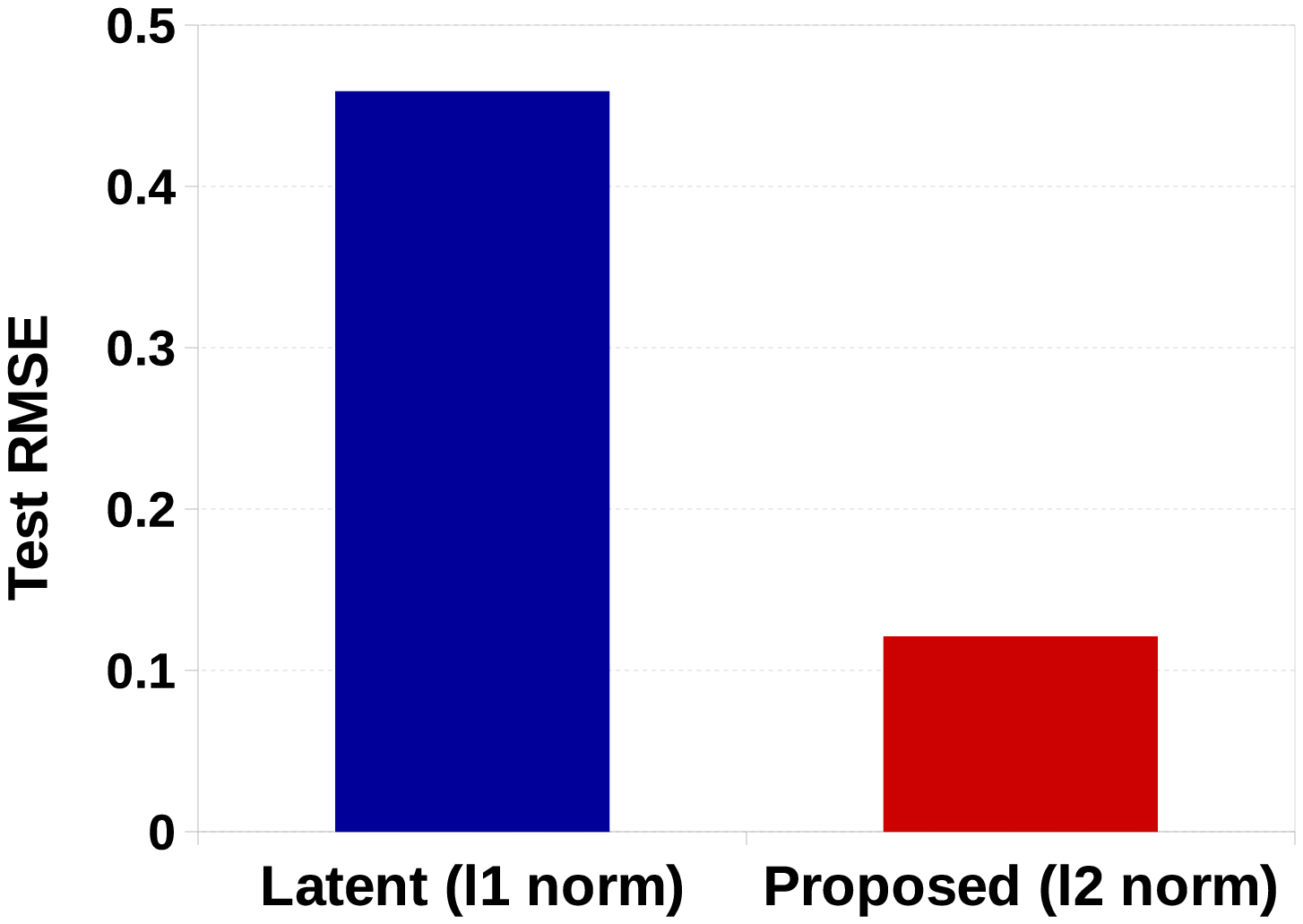}\\
{\small (d) Baboon }
\end{minipage}
\end{tabular}
\caption{\small (a) \& (b) Relative sparsity of each tensor in the mixture of tensors for Ribeira and Baboon datasets. Our proposed formulation learns a $\ell_2$-norm based non-sparse combination of tensors; (c) \& (d)  show that the proposed non-sparse combination obtain better generalization performance on both the datasets.}
\label{fig:sparsity}
\end{figure}

\textbf{Trace norm regularized tensor completion formulations.} The works \cite{liu2013a,tomioka2013a,Signoretto2014a,paredes2013a,cheng2016a} discuss the {\it overlapped trace norm} regularization for tensor learning. The overlapped trace norm is motivated as a convex proxy for minimizing the \emph{Tucker} (multilinear) rank of a tensor. The overlapped trace norm is defined as: $R(\ten{W})\coloneqq \sum_{k=1}^K\norm{\ten{W}_{k}}_* $, where $\ten{W}_{k}$ is the mode-$k$ matrix unfolding of the tensor $\ten{W}$ \cite{kolda2009a} and $\norm{\cdot}_*$ denotes the trace norm regularizer. 
$\ten{W}_k$ is a $n_k\times \Pi_{j\neq k}\, n_j$ matrix obtained by concatenating mode-$k$ fibers (column vectors) of the form $\ten{W}_{(i_1,\ldots,i_{k-1},:,i_{k+1},\ldots,i_K)}$ \cite{kolda2009a}. 

{\it Latent trace norm} is another convex regularizer used for low-rank tensor learning~\cite{tomioka2010a,tomioka2011a,tomioka2013a,wimalawarne2014a,guo2017a}. In this setting, the tensor $\ten{W}$ is modeled as {\it sum} of $K$ (unknown) tensors $\ten{W}^{(1)},\ldots,\ten{W}^{(K)}$ such that $\ten{W}_{k}^{(k)}$ are low-rank matrices. The latent trace norm is defined as: 
\begin{equation}\label{eq:q_latent_norm}
R(\ten{W})\coloneqq { \underset{\sum_{k=1}^{K}\ten{W}^{(k)}=\ten{W};\text{~} \ten{W}^{(k)} \in \mathbb{R}^{n_1 \times \ldots \times n_K}}{\inf}}\ \textstyle\sum_{k = 1}^{K} \|\ten{W}^{(k)}_{k}\|_*  ,
\end{equation}
A variant of the latent trace norm  ($\|\ten{W}^{(k)}_{k}\|_{*}$ scaled by $1/\sqrt{n_k}$) is analyzed in \cite{wimalawarne2014a}. 
Latent trace norm and its scaled variant achieve better recovery bounds than overlapped trace norm  \cite{tomioka2013a,wimalawarne2014a}. 
Recently, \cite{guo2017a} proposed a scalable latent trace norm based Frank-Wolfe algorithm for tensor completion. 




The latent trace norm (\ref{eq:q_latent_norm}) corresponds to the sparsity inducing $\ell_1$-norm penalization across $\|\ten{W}^{(k)}_{k}\|_*$. 
Hence, it learns $\ten{W}$ as a sparse combination of $\ten{W}^{(k)}$. In case of high sparsity, it may result in selecting only one of the tensors $\ten{W}^{(k)}$ as $\ten{W}$, i.e., $\ten{W}=\ten{W}^{(k)}$ for some $k$, in which case $\ten{W}$ is essentially learned as a low-rank matrix. In several real-world applications, tensor  data cannot be mapped to a low-rank matrix structure and they require a higher order structure. Therefore, we propose a regularizer which learns a non-sparse combination of $\ten{W}^{(k)}$. Non-sparse norms have led to better generalization performance in other machine learning settings~\cite{cortes09a,suzuki2011a,mjaw14a}. 

We show the benefit of learning a non-sparse mixture of tensors as against a sparse mixture on two datasets: Ribeira  and Baboon (refer Section~\ref{sec:exp} for details). Figures \ref{fig:sparsity}(a) and \ref{fig:sparsity}(b) show the relative sparsity of the optimally learned tensors in the mixture as learned by the $\ell_1$-regularized latent trace norm based model (\ref{eq:q_latent_norm})  \cite{tomioka2013a,wimalawarne2014a,guo2017a} versus the proposed $\ell_2$-regularized model (discussed in Section \ref{sec:formulation}). The relative sparsity for each $\ten{W}^{(k)}$ in the mixture is computed as $\|\ten{W}^{(k)}\|_F/\sum_{k} \|\ten{W}^{(k)}\|_F$. In both the datasets, our model learns a non-sparse combination of tensors, whereas the latent trace norm based model learns a highly skewed mixture of tensors. 
The proposed non-sparse tensor combination also leads to better generalization performance, as can be observed in the Figures \ref{fig:sparsity}(c) and \ref{fig:sparsity}(d). 
In the particular case of Baboon dataset, the latent trace norm essentially learns $\ten{W}$ as a low-rank matrix ($\ten{W}=\ten{W}^{(3)}$) and consequently obtains poor generalization.

\textbf{Other tensor completion formulations.} Other approaches for low-rank tensor completion include tensor decomposition methods like Tucker and CP \cite{kolda2009a,cichocki2017a,cichocki2017b}. They generalize the  notion of singular value decomposition of matrices to tensors. Recently, \cite{Kressner2014a} exploits the Riemannian geometry of fixed multilinear rank to learn factor matrices and the core tensor. They propose a computationally efficient non-linear conjugate gradient method for optimization over manifolds of tensors of fixed multilinear rank. \cite{Kasai2016a} further propose an efficient preconditioner for low-rank tensor learning with the Tucker decomposition. \cite{Zhao2015} propose a Bayesian probabilistic CP model for performing tensor completion. Tensor completion algorithms based on tensor {\it tubal-rank} have been recently proposed in \cite{Zhang2014,liu2016}.

\section{Non-sparse latent trace norm, duality, and novel formulations for low-rank tensor completion}
\label{sec:formulation}
We propose the following formulation for learning the low-rank tensor $\ten{W}$ 
\begin{equation}\label{eq_primal}
{\begin{array}{lll}
 \min\limits_{ \ten{W}^{(k)} \in \mathbb{R}^{n_1\times \ldots \times n_K}}  \  \norm{\spa{\ten{W}} - \spa{\ten{Y}}}_F^2 
  + \sum\limits_{k} \frac{1}{\lambda_k} \norm{\ten{W}^{(k)}_{k}}_*^2, 
\end{array}
}
\end{equation}
where $\ten{W}=\sum_{k}\ten{W}^{(k)}$ is the learned tensor. It should be noted that the proposed regularizer in (\ref{eq_primal}) employs the $\ell_2$-norm over $\|\ten{W}^{(k)}_{k}\|_*$. In contrast, the latent trace norm regularizer  (\ref{eq:q_latent_norm}) has the $\ell_1$-norm over $\|\ten{W}^{(k)}_{k}\|_*$.


While the existing tensor completion approaches \cite{Kasai2016a,guo2017a,Kressner2014a,liu2013a,tomioka2013a,Signoretto2014a} mostly discuss a primal formulation similar to (\ref{eqn1}), we propose a novel dual framework for our analysis. The use of dual framework, e.g., in the matrix completion problem \cite{xin12a,pong10a,mjaw18a}, often leads to novel insights into the solution space of the primal problem. 

We begin by discussing how to obtain the dual formulation of (\ref{eq_primal}). Later, we explain how the insights from the dual framework motivate us to propose \emph{two novel} fixed-rank formulations for  (\ref{eq_primal}). 
As a first step, we exploit the following {\it variational characterization} of the trace norm studied in \cite[Theorem~4.1]{Argyriou2006a}. Given $\mat{X} \in \mathbb{R}^{d \times T}$, the following result holds:
\begin{equation}
\norm{\mat{X}}_*^2 = \underset{\Theta \in \spect{P}^d, {\range}(\mat{X}) \subseteq {\range}( \Theta )}{\min}  \langle \Theta^{\dag}, \mat{X}\mat{X}^\top \rangle, 
\label{eqn_var}
\end{equation}
where $\spect{P}^d$ denotes the set of $d \times d$ positive semi-definite matrices with \emph{unit trace}, $\Theta^{\dag}$ denotes the pseudo-inverse of $\Theta$, ${\range}(\Theta) = \{ \Theta z : z \in \mathbb{R}^d\}$, and $\langle \cdot, \cdot \rangle$ is the inner product. The expression for optimal ${\Theta}^*$ is $\Theta^{*} = {\sqrt{\mat{X}\mat{X}^\top}}/{{\trace}(\sqrt{\mat{X}\mat{X}^\top)}}$ \cite{Argyriou2006a}, and hence the ranks of $\Theta$ and $\mat{X}$ are equal at optimality. Thus, (\ref{eqn_var}) implicitly transfers the low-rank constraint on $\bX$ (due to trace norm) to an auxiliary variable $\Theta \in \mathcal{P}^d$. 
It is well known that positive semi-definite matrix $\Theta$ with unit trace constraint implies the $\ell_1$-norm constraint on the eigenvalues of $\Theta$, leading to low-rankness of $\Theta$. 


Using the result (\ref{eqn_var}) in (\ref{eq_primal}) leads to $K$ auxiliary matrices, one $\Theta_k \in \P^{n_k}$ corresponding to every $\ten{W}^{(k)}_k$ (mode-$k$ matrix unfolding of the tensor $\ten{W}^{(k)}$). It should also be noted that $\Theta_k\in\P^{n_k}$ are low-rank matrices. We now present the following theorem that states an equivalent minimax formulation of (\ref{eq_primal}).  
\begin{theorem}\label{dual_theorem}
An equivalent minimax formulation of the problem (\ref{eq_primal}) is 
\begin{equation}\label{eq:minmax}
\minop_{\Theta_1\in \spect{P}^{n_1}, \ldots,\Theta_K\in \spect{P}^{n_K}}\ \maxop_{\ten{Z} \in \mathcal{C}}\  \  \langle \ten{Z}, \ten{Y}_\Omega \rangle 
- \frac{1}{4} \|\ten{Z}\|_F^2 
 -  \displaystyle\sum\limits_k\frac{\lambda_k}{2} \langle \Theta_k, \ten{Z}_k\ten{Z}_k^\top\rangle,
\end{equation}
where $\ten{Z}$ is the dual tensor variable corresponding to the primal problem (\ref{eq_primal}) and $\ten{Z}_k$ is the mode-$k$ unfolding of $\ten{Z}$. The set $\mathcal{C} \coloneqq \{\ten{Z} \in \mathbb{R}^{n_1 \times \ldots \times n_K}: \ten{Z}_{(i_1,\ldots,i_K)}=0,  (i_1,\ldots,i_K) \notin \Omega \}$ constrains $\ten{Z}$ to be a sparse tensor with $|\Omega|$ non-zero entries. 

Furthermore, let $\{\Theta_1^*, \ldots,\Theta_K^*,\ten{Z}^*\}$ be the optimal solution of (\ref{eq:minmax}). The optimal solution of (\ref{eq_primal}) is given by $\ten{W}^{*}=\sum_k {\ten{W}^{(k)*}}$, where ${\ten{W}^{(k)*}} = \lambda_k (\ten{Z}^* \times_k \Theta_k^*)\ \forall k$ and  $\times_k$ denotes the tensor-matrix multiplication along mode $k$.  
\end{theorem}

\begin{proof}
From using the auxiliary $\Theta_k$s in \eqref{eq_primal}, we obtain the following formulation:
\begin{equation}
\underset{\ten{W}^{(k)}, k \in \{1,\ldots,K\}}{\min} \norm{\spa{(\sum_{k =1}^{K} \ten{W}^{(k)})} - \spa{\ten{Y}}}_F^2 + \sum_{k=1}^{K}\frac{1}{2\lambda_k}\langle \Theta^{\dag}_k, \ten{W}^{(k)}_k {\ten{W}^{(k)}_k}^{\top}\rangle .
\label{supp:eq_theta_primal}
\end{equation}
For deriving the dual, we introduce new variables $\mat{A}_k, k \in \{1,\ldots, K\}$ which satisfy the constraints $\mat{A}_k = \ten{W}^{(k)}_k$. We now introduce the dual variables $\Lambda_k \in \mathbb{R}^{n_k \times n_1\ldots n_K/n_k}$ corresponding to those additional constraints. The Lagrangian $L$ of \eqref{supp:eq_theta_primal} is given as:
\begin{multline}
L(\ten{W}^{(1)}, \ldots, \ten{W}^{(K)}, \mat{A}_1, \ldots, \mat{A}_K, \Lambda_1, \ldots, \Lambda_K) =   \norm{\spa{(\sum_{k}^{K} \ten{W}^{(k)})} - \spa{\ten{Y}}}_F^2 +
\\ \sum_{k=1}^{K}\frac{1}{2\lambda_k}\langle \Theta^{\dag}_k, \mat{A}_k \mat{A}_k^\top\rangle + \sum_{k=1}^{K} \langle \Lambda_k, (\ten{W}_k^{(k)} - \mat{A}_k) \rangle .
\label{supp:eq_langrange}
\end{multline}
The dual function of \eqref{supp:eq_theta_primal} is defined as:
\begin{equation}
\underset{k \in \{1,\ldots, K \}}{\underset{\mat{A}_k \in \mathbb{R}^{n_k \times n_1 \ldots n_K/n_k},}{\underset{\ten{W}^{(k)} \in \mathbb{R}^{n_1 \times \ldots \times n_K},}{\min}}} L(\ten{W}^{(1)}, \ldots, \ten{W}^{(K)}, \mat{A}_1, \ldots, \mat{A}_K, \Lambda_1, \ldots, \Lambda_K).
\label{supp:eq_L_dual}
\end{equation}
By minimizing $L$ with respect to $\ten{W}^{(k)}$ and $\mat{A}_k$, for all $k \in \{1, \ldots, K\}$, we arrive at the  following conditions:
\begin{equation}
[\Lambda_k]_k= 2(\spa{\ten{Y}} - \spa{(\sum_{k =1}^{K} \ten{W}^{(k)})}),
\label{supp:eq_derive1}
\end{equation}
\begin{equation}
\mat{A}_k = \lambda_k   \Theta_k \Lambda_k,
\label{supp:eq_derive2}
\end{equation}
where $[\mat{P}]_k$ represents mode-$k$ folding of a matrix $\mat{P}$ into a tensor.

From \eqref{supp:eq_derive1}, it can be seen that all $[\Lambda_k]_k$ are equal, which we represent with $\ten{Z}$. Therefore \eqref{supp:eq_derive1} and \eqref{supp:eq_derive2} can written as:
\begin{equation}
 \ten{Z} = 2(\spa{\ten{Y}} - \spa{(\sum_{k =1}^{K} \ten{W}^{(k)})}),
\label{supp:eq_derive11}
\end{equation}
\begin{equation}
\mat{A}_k = \lambda_k  \ten{Z} \times_k \Theta_k,
\label{supp:eq_derive21} 
\end{equation}

From \eqref{supp:eq_derive11}, it is clear that $\ten{Z}_{(i_1, \ldots, i_K)}$ is non-zero only for $(i_1, \ldots, i_K) \in \Omega$, i.e., $\ten{Z}$ is a sparse tensor, thereby ensuring the constraint $\ten{Z} \in \mathcal{C}$ is satisfied. Using \eqref{supp:eq_derive11} and \eqref{supp:eq_derive21} in \eqref{supp:eq_L_dual} gives the dual formulation and hence proving the theorem.
\end{proof}

\begin{remark}
Theorem \ref{dual_theorem} shows that the optimal solutions $\ten{W}^{(k)*}$ for all $k$ of (\ref{eq_primal}) are completely characterized by a {\it single} sparse tensor $\ten{Z}^*$ and $K$ low-rank positive semi-definite matrices $\{\Theta^{*}_1,\dots,\Theta^{*}_K\}$. It should be noted that such a novel relationship of $\ten{W}^{(k)*}$ (for all $k$) with each other is not evident from the formulation (\ref{eq_primal}). 
\end{remark}

We next present the following result related to the form of the optimal solution of (\ref{eq_primal}). 
\begin{corollary}\label{cor:representer}
(Representer theorem) The optimal solution of the primal problem (\ref{eq_primal}) admits a representation of the form: ${\ten{W}^{(k)*}} = \lambda_k (\ten{Z} \times_k \Theta_k)\ \forall k$, where $\ten{Z} \in \mathcal{C}$ and $\Theta_k\in \spect{P}^{n_k}$. 
\end{corollary}


Instead of solving the minimax problem (\ref{eq:minmax}) directly, we solve the following equivalent min optimization problem: 
\begin{equation}
 \underset{\theta \in \spect{P}^{n_1} \times \ldots \times \spect{P}^{n_K}}{\min} \ f(\theta),
 \label{dual1}
\end{equation}
where $\theta = (\Theta_1, \ldots, \Theta_K)$ and $ f : \spect{P}^{n_1} \times \ldots \times \spect{P}^{n_K} \rightarrow \mathbb{R} $ is the convex function
{
\begin{equation} \label{dual2}
\begin{array}{lll}
f (\theta)  \coloneqq   \displaystyle\underset{\ten{Z} \in \mathcal{C}}{\max} \  \langle \ten{Z}, \ten{Y}_\Omega \rangle - \displaystyle\frac{1}{4} \norm{\ten{Z}}_F^2
- \displaystyle\sum\limits_k\frac{\lambda_k}{2} \langle \Theta_k, \ten{Z}_k\ten{Z}_k^\top\rangle.
 \end{array}
\end{equation}
}The problem formulation (\ref{dual1}) allows to decouple the structures of the min and max problems in (\ref{eq:minmax}).

As discussed earlier in the section, the optimal $\Theta_k ^{*}\in\spect{P}^{n_k}$ is a low-rank positive semi-definite matrix for all $k$. 
In spite of the low-rankness of the optimal solution, an algorithm for (\ref{eq:minmax}) need not produce intermediate iterates that are low rank. From the perspective of large-scale applications, this observation as well as other computational efficiency concerns discussed below motivate to exploit a fixed-rank parameterization of $\Theta_k$ for all $k$.

\subsection*{Fixed-rank parameterization of $\Theta_k$} 
We propose to explicitly constrain the rank of $\Theta_k$ to $r_k$ as follows:
\begin{equation} \label{eq:parameterization}
\Theta_k = \uu{k},
\end{equation} 
where $\mat{U}_k \in \spect{S}^{n_k}_{r_k}$ and $ \spect{S}^{n}_{r} \coloneqq \{\mat{U}\in \mathbb{R}^{n \times r}: \norm{\mat U}_F = 1\}$. In large-scale tensor completion problems, it is common to set $r_k \ll n_k$, where the fixed-rank parameterization (\ref{eq:parameterization}) of $\Theta_k$ has a two-fold advantage. First, the search space dimension of (\ref{dual1}) drastically reduces from $n_k((n_k + 1)/2 -1)$, which is {\it quadratic} in tensor dimensions, to $n_k r_k  -1-r_k(r_k-1)/2$, which is {\it linear} in tensor dimensions \cite{journe2010a}. Second, enforcing the constraint $ \mat{U}_k \in \spect{S}^{n_k}_{r_k}$ costs $O(n_k r_k)$, which is {\it linear} in tensor dimensions and is computationally much cheaper than enforcing $\Theta_k \in \P^{n_k}$ that costs $O(n_k^3)$. 


\subsection*{Fixed-rank dual formulation} 
A first formulation is obtained by employing the parameterization (\ref{eq:parameterization}) directly in the problems (\ref{dual1}) and (\ref{dual2}). We subsequently solve the resulting problem as a minimization problem as follows 
\begin{equation}  \label{udual1}
\text{Problem\ \ } \ten{D}: \qquad  \displaystyle\min\limits_{u \in \spect{S}^{n_1}_{r_1} \times \ldots \times \spect{S}^{n_K}_{r_K}} \quad g(u),
\end{equation}
where $u= (\mat{U}_1, \ldots, \mat{U}_K)$ and $ g : \spect{S}^{n_1}_{r_1} \times \ldots \times \spect{S}^{n_K}_{r_K}\rightarrow \mathbb{R} $ is the function
{
\begin{equation} \label{udual2}
\begin{array}{lll}
g (u)  \coloneqq 
   \displaystyle\underset{\ten{Z} \in \mathcal{C}}{\max} \  \langle \ten{Z}, \ten{Y}_\Omega \rangle - \displaystyle\frac{1}{4} \norm{\ten{Z}}_F^2
- \displaystyle\sum\limits_k\frac{\lambda_k}{2} \norm{{\mat U}_k^\top \ten{Z}_k }_F^2.
 \end{array}
\end{equation}
}It should be noted that though (\ref{udual1}) is a non-convex problem in $u$, the optimization problem in (\ref{udual2}) is {\it strongly} convex in $\ten{Z}$ for a given $u$ and has {\it unique} solution. 

\subsection*{Fixed-rank least-squares formulation}

The second formulation is motivated by the representer theorem (Corollary \ref{cor:representer}) and the fixed-rank parameterization (\ref{eq:parameterization}). However, instead of solving (\ref{eq_primal}), we take a more practical approach of solving the following low-rank
\begin{equation} \label{eqn:primal_main}
 \text{Problem\ \ } \ten{P}:   \qquad  \underset{\ten{Z} \in \mathcal{C}}{\min\limits_{u \in \spect{S}^{n_1}_{r_1} \times \ldots \times \spect{S}^{n_K}_{r_K}}} \quad   h(u, \ten{Z})
\end{equation}
and $ h : \spect{S}^{n_1}_{r_1} \times \ldots \times \spect{S}^{n_K}_{r_K} \times \mathcal{C}\rightarrow \mathbb{R} $ is
\begin{equation}{
\label{eqn:primal_2}
 h(u, \ten{Z}) \coloneqq \norm{(\sum_{k} \lambda_k (\ten{Z}\times_k \mat{U}_k \mat{U}_k^{\top}))_{\Omega} - \ten{Y}_{\Omega}}_F^2.}
\end{equation}
where $\ten{W}=\sum_{k} \ten{W}^{(k)}$ and  $\ten{W}^{(k)}=\lambda_k \ten{Z}\times_k \mat{U}_k \mat{U}_k^{\top}$. {It should be noted that the objective function in (\ref{eqn:primal_main}) does not have an explicit regularizer as in (\ref{eq_primal}) to ensure non-sparse $\ten{W}^{(k)}$ and low-rank $\ten{W}_k^{(k)}$. However, the regularizer is {\it implicit} since we employed the representer theorem and the fixed-rank parameterization ($\ten{Z}$ is common for all $\ten{W}^{(k)}$ and $\mat{U}_k \in \spect{S}^{n_k}_{r_k}$). }


\section{Optimization algorithms}
\label{sec:opt}

\begin{figure}[t]
\center
\includegraphics[scale=0.5]{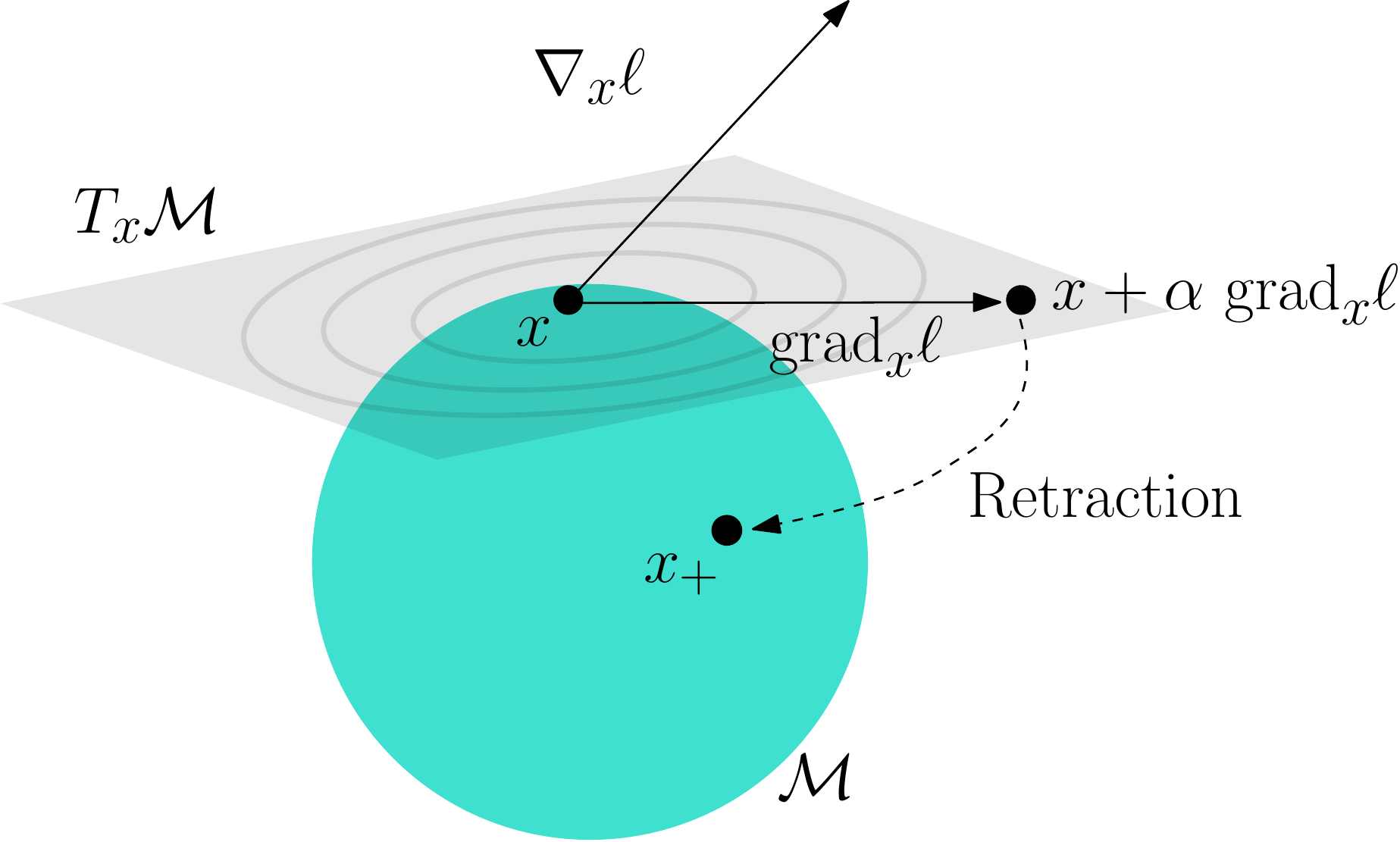}
\caption{A schematic view of the gradient descent algorithm for minimizing a function $\ell$ on a manifold $\mathcal{M}$. The {\it linearization} of $\mathcal{M}$ at $x$ is characterized by the tangent space $T_x \mathcal{M}$. The Riemannian gradient $\grad_x \ell$ is the steepest-descent direction of $\ell$ (in the tangent space $T_x \mathcal{M}$), which is derived from the standard gradient $\nabla_x \ell(x)$ (which need not lie in $T_x \mathcal{M}$). Following $\grad_x \ell$ with step size $\alpha$ leaves the manifold that is subsequently pulled back onto the manifold with the retraction operation to ensure strict feasibility. For the trust-region algorithm, additional care is taken to exploit the second-order geometry of the manifold.}
\label{fig:manifoldOpt}
\end{figure}

In this section we discuss the optimization algorithms for $\ten{D}$~(\ref{udual1}) and $\ten{P}$~(\ref{eqn:primal_main}). Both of these are of the type
\begin{equation}
 \label{eqn:gen_opt}
 \min\limits_{x \in \mathcal{M}}\quad \ell(x),
\end{equation}
where $\ell: \mathcal{M} \rightarrow \mathbb{R}$ is a smooth loss and $\mathcal{M}$ is the constraint set. For Problem $\ten{D}$, $\mathcal{M} \coloneqq \spect{S}^{n_1}_{r_1} \times \ldots \times \spect{S}^{n_K}_{r_K}$ and Problem $\ten{P}$, $\mathcal{M} \coloneqq \spect{S}^{n_1}_{r_1} \times \ldots \times \spect{S}^{n_K}_{r_K} \times \mathcal{C}$.

In order to propose numerically efficient algorithms for optimization over $\mathcal{M}$, we exploit the particular structure of the set $\spect{S}^{n}_{r}$, which is known as the {\it spectrahedron} manifold \cite{journe2010a}. The spectrahedron manifold has the structure of a compact Riemannian quotient manifold \cite{journe2010a}. Consequently, optimization on the spectrahedron manifold is handled in the Riemannian optimization framework. This allows to exploit the rotational invariance of the constraint $\norm{\mat{U}}_F = 1$ naturally. The Riemannian manifold optimization framework embeds the constraint $\norm{\mat{U}}_F = 1$ into the search space, thereby translating the constrained optimization problem  into {\it unconstrained} optimization problem over the {spectrahedron} manifold. The Riemannian framework generalizes a number of classical first- and second-order (e.g., the conjugate gradient and trust-region algorithms) Euclidean algorithms to manifolds and provides concrete convergence guarantees \cite{edelman98a,absil08a,sato17a,zhang16a,sato13a}. The work \cite{absil08a} in particular shows a systematic way of implementing trust-region (TR) algorithms on quotient manifolds. A full list of optimization-related ingredients and their matrix characterizations for the spectrahedron manifold $\spect{S}^{n}_{r}$ is in Section \ref{supp:spec_opt}. Overall, the constraint $\mathcal{M}$ is endowed a Riemannian structure. Figure~\ref{fig:manifoldOpt} depicts a general overview of a gradient-based algorithm on $\mathcal{M}$.

We implement the Riemannian TR (second-order) algorithm for (\ref{eqn:gen_opt}). To this end, we require the notions of the \textit{Riemannian gradient} (the first-order derivative of the objective function on the manifold), the \textit{Riemannian Hessian} along a search direction (the \textit{covariant} derivative of the Riemannian gradient along a tangential direction on the manifold), and the \textit{retraction} operator which ensures that we always stay on the manifold (i.e., maintain strict feasibility). The Riemannian gradient and Hessian notions require computations of the standard (Euclidean) gradient  $\nabla_x \ell(x)$ and the directional derivative of this gradient along a given search direction $v$ denoted by ${\rm D}\nabla_x \ell(x)[v]$, the expressions of which for $\ten{D}$~(\ref{udual1}) and $\ten{P} ($\ref{eqn:primal_main}) are shown in Lemma \ref{u_lemma} and Lemma \ref{primal_grad_lemma}, respectively.
\begin{lemma}\label{u_lemma}
Let $\hat{\ten{Z}}$  be the optimal solution of the convex problem (\ref{udual2}) at $u \in \spect{S}^{n_1}_{r_1} \times \ldots \times \spect{S}^{n_K}_{r_K}$. Let $\nabla_{u} g$ denote the gradient of $g(u)$ at $u$ and ${\mathrm{D}}\nabla_{u} g[v]$ denote the directional derivative of the gradient $\nabla_{u} g$ along $v \in \mathbb{R}^{n_1 \times r_1} \times \ldots \times  \mathbb{R}^{n_K \times r_K} $. Let $\dot{\ten{Z}_k}$ denote the directional
 derivative of $\ten{Z}_k$ along $v$ at $\hat{\ten{Z}_k}$. Then,
   \begin{equation*}
   \begin{array}{lll}
   \nabla_{u} g &= (- \lambda_1 \hat{\ten{Z}}_1{\hat{\ten{Z}}_1}^\top \mat{U}_1, \ldots,  - \lambda_K \hat{\ten{Z}}_K\hat{\ten{Z}}_K^\top \mat{U}_K) \\
    {\mathrm{D}}\nabla_{u} g [v] &= (-\lambda_1\mat{A}_1, \ldots, -\lambda_K\mat{A}_K),
    \end{array}
    \end{equation*}   
where $\mat{A}_k =   \hat{\ten{Z}}_k\hat{\ten{Z}}_k^\top \mat{V}_k + 
  {\rm symm} ( {\dot{\ten{Z}}}_k {\hat{\ten{Z}}}_k^\top) \mat{U}_k $ and ${\rm symm}$ extracts the symmetric factor of a matrix, i.e., ${\rm symm}({\bf \Delta}) =\allowbreak ({\bf \Delta} + {\bf \Delta}^\top)/2$.
  \end{lemma}

\begin{algorithm*}[tb]
  \caption{Proposed Riemannian trust-region algorithm for (\ref{udual1}) and (\ref{eqn:primal_main}).\label{alg:proposed_algorithms} }
  {
  \begin{tabular}{l | l}
   \multicolumn{2}{l}{{\bfseries Input:}  $\ten{Y}_\Omega$, rank $(r_1,\ldots, r_K)$, {regularization parameter $\lambda$}, and tolerance $\epsilon$. }\\
   \multicolumn{2}{l}{ {\bfseries Initialize : } $u \in \mathcal{M}$. }\\
   \multicolumn{2}{l}{ {\bfseries repeat}}\\   
  \multicolumn{2}{l}{\ \ \ \  \textbf{1:} Compute the gradient $ \nabla_{u} \ell$ as given in Lemma \ref{u_lemma} for Problem $\ten{D}$ and Lemma \ref{primal_grad_lemma} for Problem $\ten{P}$.}   \\
  \multicolumn{2}{l}{\ \ \ \ \textbf{2:} Compute the search direction which minimizes the trust-region subproblem. It makes use }\\
  \multicolumn{2}{l}{$\qquad$ of $\nabla_u \ell $ and its directional derivative presented in Lemma \ref{u_lemma} and Lemma \ref{primal_grad_lemma} for $\ten{D}$ and $\ten{P}$,}\\ 
  \multicolumn{2}{l}{$\qquad$ respectively. }\\ 
  \multicolumn{2}{l}{\ \ \ \ \textbf{3:} {Update $u$ with the retraction step to maintain strict feasibility on $\mathcal{M}$}. }\\
  \multicolumn{2}{l}{$\qquad$ Specifically for the spectrahedron manifold, $\mat{U}_k \leftarrow (\mat{U}_k+\mat{V}_k)/\norm{\mat{U}_k + \mat{V}_k}_F$,}\\ 
    \multicolumn{2}{l}{$\qquad$ where $\mat{V}_k$ is the search direction.}\\ 
   \multicolumn{2}{l}{{\bfseries until} $\norm{\nabla_u \ell}_F<\epsilon$.}\\
   \multicolumn{2}{l}{{\bfseries Output: } $u^*$}\\
  \end{tabular}
}
\end{algorithm*}

  \begin{proof}
 The gradient is computed by employing the Danskin's theorem \cite{Bertsekas99,Bonnans00}. The directional derivative of the gradient with respect to $\mat{U}$ follows directly from the chain rule keeping $\ten{Z}$ fixed (to the optimal solution of (\ref{udual2}) for given $\mat{U}$).
  \end{proof}

A key requirement in Lemma \ref{u_lemma} is solving (\ref{udual2}) for $\hat{\ten{Z}}$ computationally efficiently for a given $u =  (\mat{U}_1, \ldots, \mat{U}_K) $. It should be noted that (\ref{udual2}) has a closed-form sparse solution, which is equivalent to solving the {\it linear} system 
\begin{equation}\label{eq:linear_system}{
\begin{array}{ll}
\hat{\ten{Z}}_\Omega+ \sum_k \lambda_k (\hat{\ten{Z}}_\Omega \times_k \mat{U}_k {\mat{U}_k^\top})_\Omega  = \ten{Y}_\Omega.
\end{array}}
\end{equation}
Solving the linear system (\ref{eq:linear_system}) in a {\it single} step is computationally expensive (it involves the use of Kronecker products, vectorization of a sparse tensor, and a matrix inversion). Instead, we use an {\it iterative} solver that exploits the sparsity in the variable $\ten Z$ and the factorization form $\uu{k}$ efficiently. Similarly, given $\hat{\ten{Z}}$ and $v$, $\dot{\ten{Z}}$ can be computed by solving
\begin{equation}\label{eq:dot_linear_system}{
\begin{array}{ll}
\dot{\ten{Z}}_\Omega+ \sum_k \lambda_k (\dot{\ten{Z}}_\Omega \times_k \mat{U}_k {\mat{U}_k^\top})_\Omega  = 
 - \sum_k \lambda_k (\hat{\ten{Z}}_\Omega \times_k (\mat{V}_k {\mat{U}_k^\top} + \mat{U}_k {\mat{V}_k^\top}) )_\Omega .
\end{array}}
\end{equation}

\begin{lemma}\label{primal_grad_lemma}
The gradient of $h(u, \ten{Z})$ in (\ref{eqn:primal_2}) at $(u, \ten{Z})$ and their directional derivatives along $(v, \dot{\ten{Z}})$, where $v \in \mathbb{R}^{n_1 \times r_1} \times \ldots \times  \mathbb{R}^{n_K \times r_K} $ and $\dot{\ten{Z}} \in \mathcal{C}$ are given as follows.
   \begin{equation*}
    {
    \begin{array}{llll}
     \nabla_{u}h \quad =  (\lambda_1 {\rm symm}(\ten{R}_1\ten{Z}_1^\top)\mat{U}_1, \ldots, \lambda_K {\rm symm}(\ten{R}_K\ten{Z}_K^\top)\mat{U}_K)\\
     \nabla_{\ten{Z}}h \quad =   \sum_{k} \lambda_k (\ten{R} \times_k \mat{U}_k\mat{U}_k^\top)_{\Omega} \\
     {\mathrm{D}}\nabla_{u}h[(v, \dot{\ten{Z}})] = (\lambda_1\mat{B}_1, \dots, \lambda_K\mat{B}_K)\\
     {\mathrm{D}}\nabla_{\ten{Z}}h[(v, \dot{\ten{Z}})] =  \sum_{k} \lambda_k( \dot{\ten{R}}\times_k\mat{U}_k\mat{U}_k^\top + \ten{R}\times_k {\rm symm}(\mat{U}_k\mat{V}_k^\top))_{\Omega},
    \end{array}}
   \end{equation*}
where 
\begin{equation}
{
\begin{array}{lll}
\ten{R} = \sum_{k} \lambda_k (\ten{Z}\times_k \mat{U}_k \mat{U}_k^{\top})_{\Omega} - \ten{Y}_{\Omega} \\
\mat{B}_k =  {\rm symm}(\dot{\ten{R}}_k\ten{Z}_k^\top + \ten{R}_k\dot{\ten{Z}}_k^\top)\mat{U}_k  + {\rm symm}(\ten{R}_k\ten{Z}_k^\top)\mat{V}_k\\
\dot{\ten{R}} = \sum_{k} \lambda_k (\dot{\ten{Z}}\times_k\mat{U}_k\mat{U}_k^\top + \ten{Z}\times_k {\rm symm}(\mat{U}_k\mat{V}_k^\top) )_{\Omega} .
\end{array}}
\end{equation} 
  \end{lemma}

\begin{proof}
The directional derivatives of the gradient with respect to $\mat{U}$ and $\ten{Z}$ follow directly from the chain rule.
\end{proof}


The Riemannian TR algorithm solves a Riemannian trust-region sub-problem in every iteration \cite[Chapter~7]{absil08a}. The TR sub-problem is a {\it second-order} approximation of the objective function in a neighborhood, solution to which does not require inverting the full Hessian of the objective function. It makes use of the gradient $\nabla_x \ell$ and its directional derivative along a search direction. The TR sub-problem is approximately solved with an iterative solver, e.g., the truncated conjugate gradient algorithm. The TR sub-problem outputs a potential update candidate for $x$, which is then accepted or rejected based on the amount of decrease in $\ell$. Algorithm \ref{alg:proposed_algorithms} summarizes the key steps of the TR algorithm for solving (\ref{eqn:gen_opt}).

\textbf{Remark 2:} In contrast to Problem $\ten{D}$, where (\ref{udual2}) is required to be solved for $\ten{Z}$ at a given $u$, in Problem $\ten{P}$ both $u$ and $\ten{Z}$ are updated simultaneously.

\subsection*{Computational and memory complexity} 

The per-iteration computational complexities of our Riemannian TR algorithms for both the formulations -- $\ten{D}$~(\ref{udual1}) and $\ten{P}$~(\ref{eqn:primal_main}) -- are shown in Table \ref{tbl:complexity}. The computational complexity scales \emph{linearly} with the number of known entries $\ten{Y}_{\Omega}$, denoted by $|\Omega|$. 

In particular, the per-iteration computational cost depends on two sets of operations. First, on manifold related ingredients like the retraction operation. Second, the objective function related ingredients like the computation of the partial derivatives. 

The manifold related operations cost $O(\sum_k n_k  r_k^2 + r_k^3)$ for $\ten{D}$~(\ref{udual1}) and $O(|\Omega| + \sum_k n_k  r_k^2 + r_k^3)$ for $\ten{P}$~(\ref{eqn:primal_main}). Specifically, the retraction on the spectrahedron manifold $\mathcal{S}_{r_k}^{n_k}$ costs $O(n_k  r_k)$ as it needs to project a matrix of size $n_k\times r_k$ on to the set $\mathcal{S}_{r_k}^{n_k}$, which costs $O(n_k  r_k)$.

The following are the computational cost for the objective function related ingredients.
\begin{itemize}\setlength\itemsep{0em}
\item \textbf{$\mat{U}_k^\top \ten{Z}_k$.} It involves computation of $n_k \times r_k$ matrix $\mat{U}_k$with mode-$k$ unfolding of a sparse $\ten Z$ with $| \Omega|$ non-zero entries. Such matrix-matrix operations are used in both $\ten{D}$~(\ref{udual1}) and $\ten{P}$~(\ref{eqn:primal_main}).The computation of $\mat{U}_k^\top \ten{Z}_k$ costs $O(|\Omega| r_k)$. It should be noted that although the dimension of $\ten{Z}_k$ is $n_k \times  { \prod_{i=1, i \neq k}^{K} n_i}$, only a maximum of $|\Omega|$ columns have non-zero entries. We exploit this property of $\ten{Z}_k$ and have a compact memory storage of $\mat{U}_k^\top \ten{Z}_k$. 

\item \textbf{Computation of the solution $\hat{\ten{Z}}$ and $\dot{\ten{Z}}$ used in $\ten{D}$~(\ref{udual1}).} We use an iterative solver for (\ref{eq:linear_system}) and (\ref{eq:dot_linear_system}), which requires computing the left hand side of (\ref{eq:linear_system}) for a given candidate $\ten{Z}$. If $m$ is the number of iterations allowed, then the computational cost is costs $O(m|\Omega| \sum_k r_k)$. 

\item \textbf{Computation of $g(u)$ in $\ten{D}$~(\ref{udual1}) and its partial derivatives.} The computation of $g(u)$ relies on the solution of (\ref{eq:linear_system}) and then explicitly computing the objective function in (\ref{udual2}). This costs $O(m|\Omega| \sum_k r_k  + K |\Omega|)$, where $m$ is the number of iterations to solve \eqref{eq:linear_system}. The computation of $\nabla_{u} g$ requires the computation of terms like $\hat{\ten{Z}}_k({\hat{\ten{Z}}_k}^\top \mat{U}_k)$, and overall it costs $O(|\Omega| \sum_k r_k)$. Given a search direction $v$, the computation of $ {\mathrm{D}}\nabla_{u} g[v]$ costs $O(|\Omega| \sum_k r_k)$.

\item \textbf{Computation of $h(u,\ten{Z})$ in $\ten{P}$~(\ref{eqn:primal_main}) and its partial derivatives.} The computation of $h(u,\ten{Z})$ costs $O(\sum_k |\Omega| r_k)$. Specifically, the computation of $\ten{R}$ costs $O(|\Omega| \sum_k r_k)$. The computations of $\nabla_{u} h$ and it directional derivatives require the computation of terms like $\mathrm{symm}(\ten{R}_k\ten{Z}_k^\top)\mat{U}_k$, which costs $O(|\Omega| \sum_k r_k)$. The same computational overhead is for $\nabla_{\ten{Z}} h$ and its directional derivative.
\end{itemize}

The memory cost for both algorithms is $O(|\Omega| + \sum_k n_k r_k)$, where $|\Omega|$ is linear in $n_k$.

\subsection*{Convergence} 
The Riemannian TR algorithms come with rigorous convergence guarantees. \cite{absil08a} discusses the rate of convergence analysis of manifold algorithms, which directly apply in our case. Specifically for trust regions, the global convergence to a first-order critical point is discussed in \cite[Section~7.4.1]{absil08a} and the local convergence to local minima is discussed in \cite[Section~7.4.2]{absil08a}. From an implementation perspective, we follow the existing approaches~\cite{Kressner2014a,Kasai2016a,guo2017a} and upper-limit the number of TR iterations.

\subsection*{Numerical implementation} 
Our algorithms are implemented in the Manopt toolbox \cite{boumal14a} in Matlab, which has off-the-shelf generic TR implementation. For efficient computation of certain sparse matrix-vector we make use of the mex support in Matlab. For solving (\ref{eq:linear_system}) and (\ref{eq:dot_linear_system}), which are used in solving $\ten{D}$~(\ref{udual1}), we rely on the conjugate gradients solver of Matlab. %

\begin{table}[tb]
 \centering
\caption{The per-iteration computational complexity for the TR algorithms for $\ten{D}$~(\ref{udual1}) and $\ten{P} ($\ref{eqn:primal_main}). For $\ten{D}$, $m$ is the number of iterations needed to solve (\ref{eq:linear_system}) and  (\ref{eq:dot_linear_system}) approximately.} \label{tbl:complexity}
{
\begin{tabular}{ll}
 \toprule
 Formulation & Computational cost \\
 \midrule
$\ten{D}$~(\ref{udual1})   & $O(m |\Omega| \sum_k r_k + \sum_k n_k  r_k^2 + \sum_k r_k^3)$\\
$\ten{P}$~(\ref{eqn:primal_main}) & $O(|\Omega| \sum_k r_k + \sum_k n_k  r_k^2 + \sum_k r_k^3)$\\
 \bottomrule
\end{tabular}}
\end{table}

\section{Experiments}
\label{sec:exp}

In this section, we evaluate the generalization performance and efficiency of the proposed algorithms against state-of-the-art in several tensor completion applications. We provide the experimental analysis of our proposed algorithms for the applications of video and hyperspectral-image completion, recommendation and link prediction. Table \ref{tbl:datasets} provides the details of all the datasets we experiment with. We compare performance of the following trace norm based algorithms.
\begin{enumerate}
\itemsep0em 
\item TR-MM: our trust-region algorithm for the proposed formulation $\ten{D}$~(\ref{udual1}).
\item TR-LS: our trust-region algorithm for the proposed formulation $\ten{P}$~(\ref{eqn:primal_main}).
\item Latent\footnote{\url{http://tomioka.dk/softwares/}} \cite{tomioka2010a}: a convex optimization based algorithm with latent trace norm.  
\item Hard \cite{Signoretto2014a}: an algorithm based on a convex optimization framework with spectral regularization for tensor learning. 
\item HaLRTC\footnote{\url{http://www.cs.rochester.edu/
u/jliu/code/TensorCompletion.zip}} \cite{liu2013a}: an ADMM based algorithm for tensor completion. 
\item FFW\footnote{\url{http://home.cse.ust.hk/~qyaoaa/}} \cite{guo2017a}: a Frank-Wolfe algorithm with scaled latent trace norm regularization. 
\end{enumerate}

We denote our algorithms for Problem $\ten{D}$~(\ref{udual1}) and Problem $\ten{P}$~(\ref{eqn:primal_main}) as {TR-MM} and {TR-LS}, respectively. For both of our algorithms, we set $\lambda_k = \lambda n_k$ for all $k$, which implies that in  $\lambda > 0$ is the only hyper-parameter that needed to be tuned. We tune $\lambda$ with $5$-fold cross-validation on the train data of the split from the range $\{10^{-3}, 10^{-2},\ldots, 10^{3}\}$.

Latent trace norm based algorithms such as {FFW} \cite{guo2017a} and {Latent} \cite{tomioka2010a}, and overlapped trace norm based algorithms such as {{HaLRTC} \cite{liu2013a}} and {{Hard} \cite{Signoretto2014a}} are the closest to our approach. FFW is a recently proposed state-of-the-art large scale tensor completion algorithm.  
Table \ref{tbl:bsline_table} summarizes the baseline algorithms. 

The proposed algorithms are implemented using Manopt toolbox \cite{boumal14a} in Matlab. All the experiments are run on an Intel Xeon CPU with 32GB RAM and 2.40GHz processor. 



\begin{table}[tb]
 \centering
 \caption{\small Summary of the datasets.} \label{tbl:datasets}
 \begin{tabular}{p{3.5cm}p{4cm}p{4cm}}
  \toprule
  Dataset & Dimensions & Task \\
  \midrule
  Ribeira & 203 $\times$ 268 $\times$ 33& Image completion\\
  Tomato & 242 $\times$ 320 $\times$ 167 & Video completion\\
  Baboon &   256 $\times$ 256 $\times$ 3  & Image completion\\
  FB15k-237 & 14541 $\times$ 14541 $\times$ 237 & Link prediction\\
  YouTube (subset) & 1509 $\times$ 1509 $\times$ 5 & Link prediction\\
  YouTube (full) & 15088 $\times$ 15088 $\times$ 5 & Link prediction\\
  ML10M & 71567 $\times$ 10681 $\times$ 731& Recommendation \\
  \bottomrule
 \end{tabular}

\end{table}

\begin{table}[tb]
\centering
\caption{ Summary of baseline algorithms.\label{tbl:bsline_table}}
\begin{tabular}{ ll } 
\toprule
  Algorithm & Modeling details \\
 \midrule
\multicolumn{2}{c}{Trace norm based algorithms}\\
 \midrule
FFW & Scaled latent trace norm (scaled $S_{1}/1$-norm) +\\ 
& Frank Wolfe optimization + basis size reduction \\

Hard& {Scaled overlapped trace norm + proximal gradient} \\ 

HaLRTC& {Scaled overlapped trace norm} + Alternating \\ 
& direction methods of multipliers  (ADMM)\\

Latent & Latent trace norm ($S_{1}/1$-norm) + ADMM\\ 
\midrule
\multicolumn{2}{c}{Other algorithms}\\
\midrule
Rprecon& Fixed multilinear rank + Riemannian CG \\
& with preconditioning\\ 

BayesCP & Bayesian CP algorithm with rank tuning\\

geomCG& {Riemannian fixed multilinear rank CG algorithm}\\ 

Topt&  {Fixed multilinear rank CG algorithm} \\ 

T-svd & Tensor tubal-rank + ADMM \\
\bottomrule
\end{tabular}
\end{table}

\begin{figure*}[tb]
\centering
{
\subfigure[Ribeira] {\includegraphics[scale=0.3]{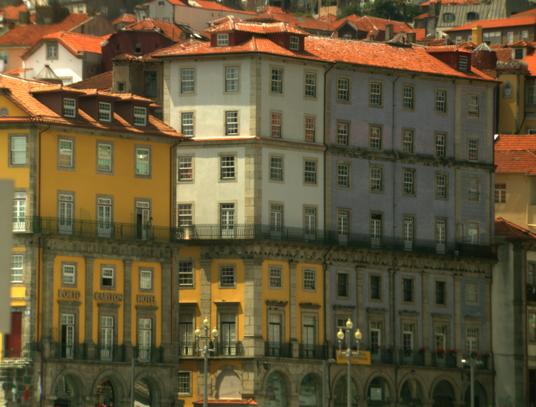}}
\subfigure[Baboon] {\includegraphics[scale=0.4755]{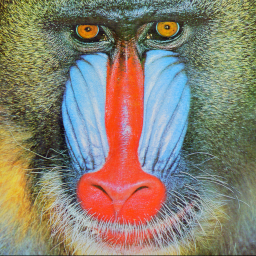}}
}
\caption{Image datasets used in our work.}
\label{fig:images}
\end{figure*}

\begin{table}[tb]
\centering
\caption{Mean test RMSE (lower is better) for hyperspectral-image completion, video completion, and recommendation problems. Our algorithms, TR-MM and TR-LS, obtain the best performance among trace norm based algorithms. `-' denotes the data set is too large for the algorithm to generate result. \label{tbl:rmse_table}}
\begin{tabular}{ lllll } 
\toprule
 & \multicolumn{1}{c}{Ribeira}  &\multicolumn{1}{c}{Tomato} & \multicolumn{1}{c}{Baboon}& \multicolumn{1}{c}{MovieLens10M} \\
 \midrule
TR-MM& $0.067\pm0.001$ & $\mathbf{0.041\pm0.001}$ & $\mathbf{0.121 \pm 0.001}$ &$0.840\pm0.001$\\ 
TR-LS & $\mathbf{0.064\pm0.002}$& $0.047\pm0.001$&$0.129 \pm 0.002$ & $0.838 \pm 0.001$\\
FFW & $0.088\pm0.001$ & $0.045\pm0.001$ & $0.133 \pm 0.001$&$0.895\pm0.001$\\ 
Rprecon& $0.083\pm0.001$ &  $0.052\pm0.000$ & $0.128 \pm 0.001$&$\mathbf{0.831\pm0.002}$\\ 
geomCG& $0.156\pm0.023$ & $0.052\pm0.000$ & $0.128 \pm 0.001$&$0.844\pm0.003$\\ 
Hard& $0.114\pm0.001$ & $0.060\pm0.001$ & $0.126 \pm 0.001$& -\\ 
Topt& $0.127\pm0.023$ & $0.102\pm0.001$ & $0.130 \pm 0.004$& -\\ 
HaLRTC& $0.095\pm0.001$&  $0.202\pm0.005$ & $0.247 \pm 0.012$& -\\ 
Latent& $0.087\pm0.001$ & $0.046\pm0.001$ & $0.459 \pm 0.002$& -\\ 
T-svd & $0.064 \pm 0.001$& $0.042\pm0.001$& $0.146 \pm 0.002 $ & -\\
BayesCP & $0.154 \pm 0.001$& $0.103\pm0.001$& $0.159 \pm 0.001$& -\\
\bottomrule
\end{tabular}
\end{table}

\begin{figure*}[tb]
\centering
{
\subfigure[Ribeira] {\includegraphics[scale=0.32]{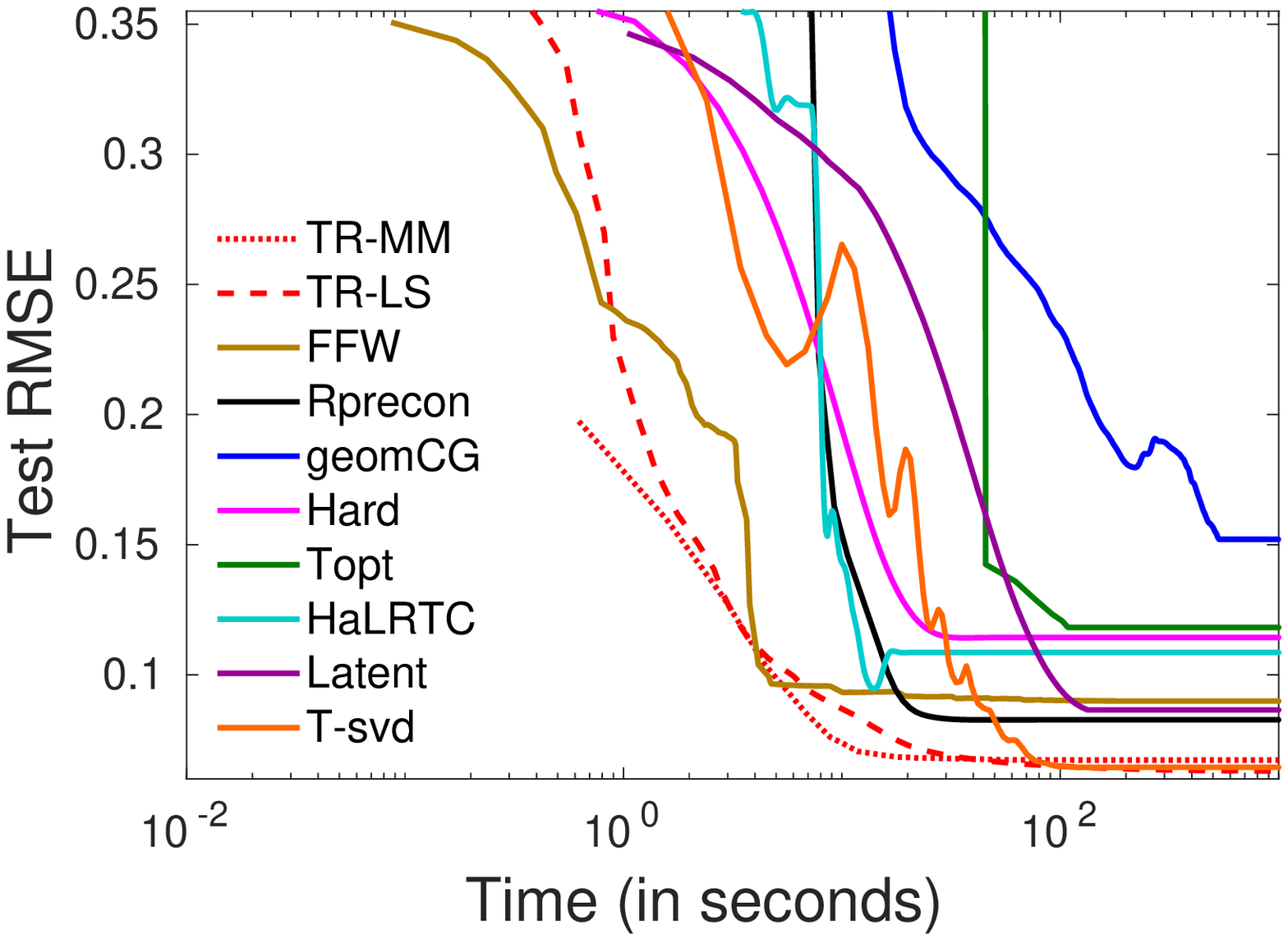}}
\subfigure[FB15k-237] {\includegraphics[scale=0.32]{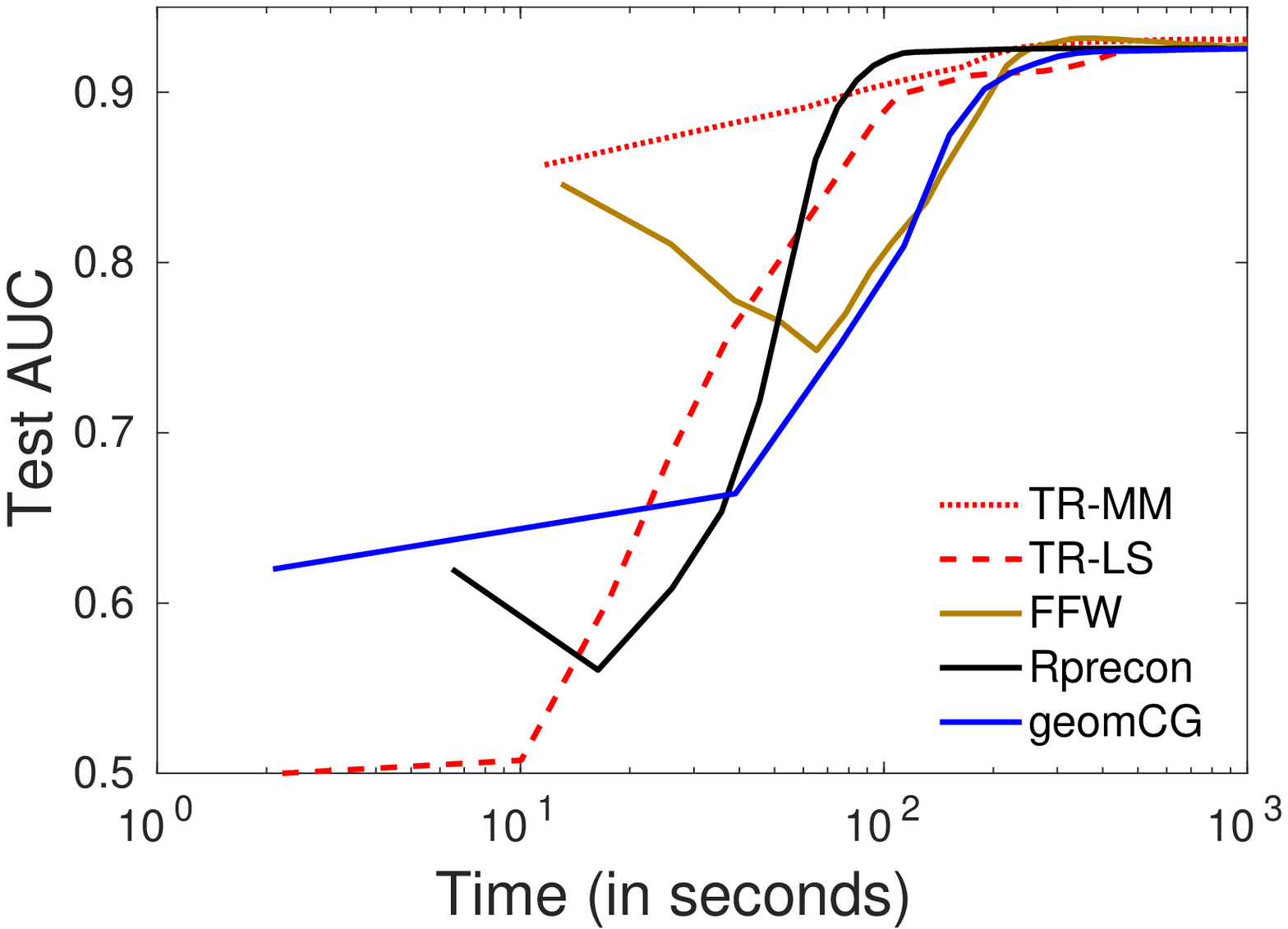}}
\subfigure[YouTube (full)] {\includegraphics[scale=0.32]{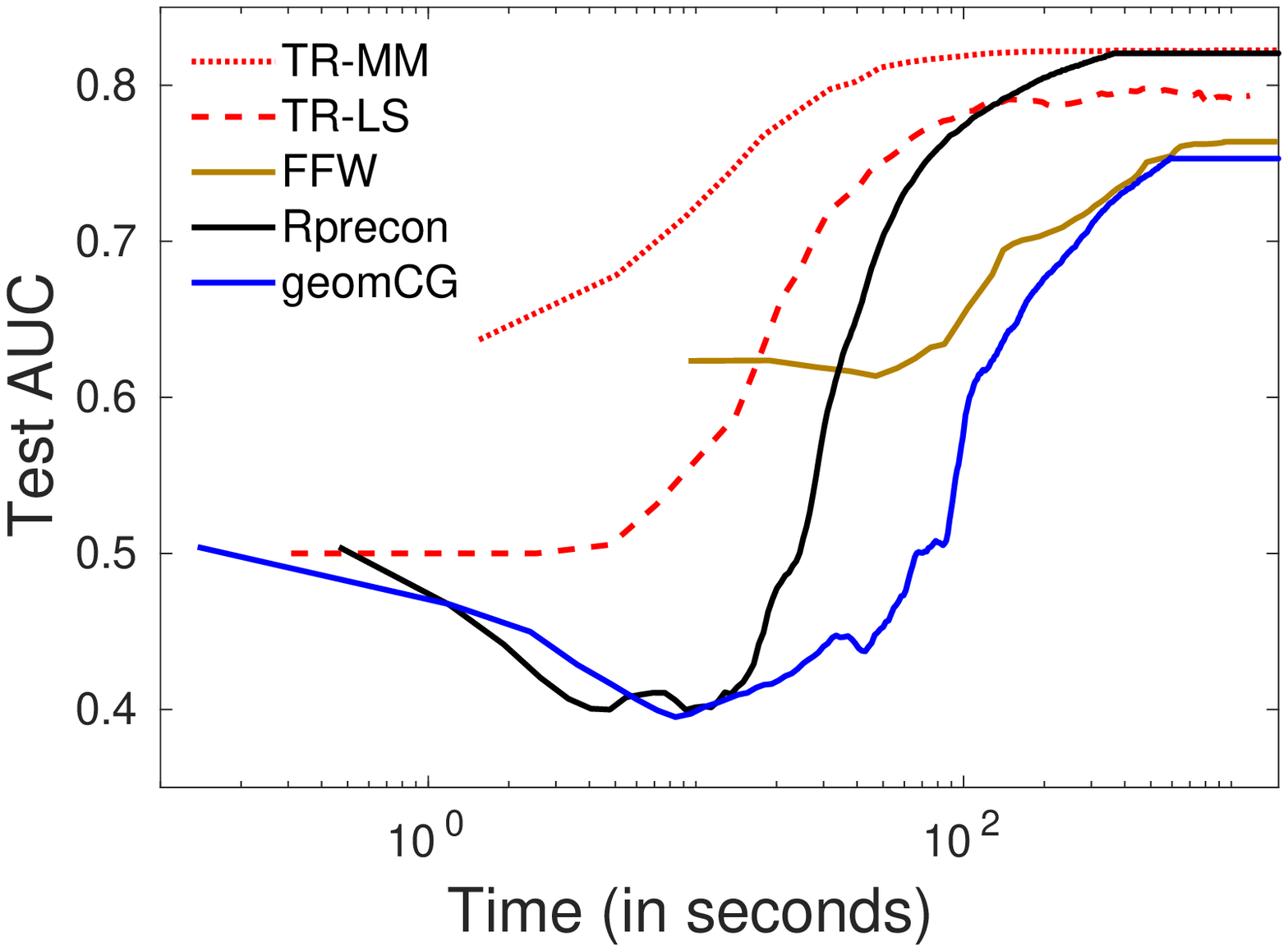}}
\subfigure[FB15k-237] {\includegraphics[scale=0.32]{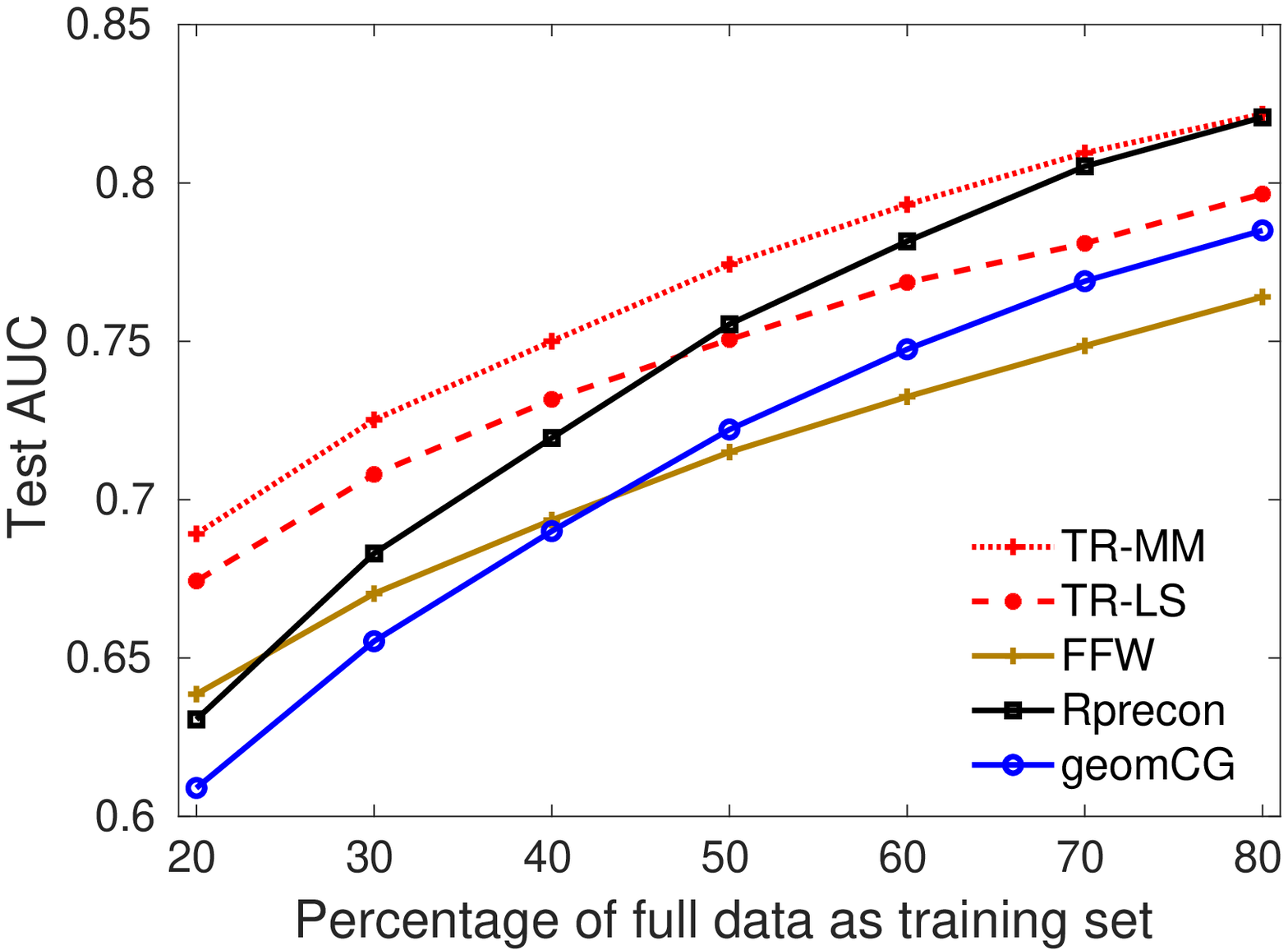}}
}
\caption{  (a) Evolution of test RMSE on Ribeira; (b) \& (c) Evolution of test AUC on FB15k-237 and YouTube, respectively. Both our algorithms, TR-MM and TR-LS, are the fastest to obtain a good generalization performance in all the three data sets; (d) Variation of test AUC as the amount of training data changes on FB15k-237. Our algorithms perform significantly better than others when the amount of training data is less.\label{fig:auc_fig}}
\end{figure*}
\subsection{Video and hyperspectral-image completion}

We work with the following data sets for predicting missing values in multi-media data like videos and hyperspectral images. \newline
\textbf{Ribeira:} A hyperspectral image data set\footnote{\url{http://personalpages.manchester.ac.uk/staff/d.h.foster/Hyperspectral_images_of_natural_scenes_04.html}} \cite{foster2004a} of size  $1017 \times 1340 \times 33$, where each slice represents a particular image measured at a different wavelength. We  re-size this tensor to $203 \times 268 \times 33$ \cite{Signoretto2014a,Kressner2014a,Kasai2016a}. Following \cite{Kasai2016a}, the incomplete tensor as training data is generated by randomly sampling $10\%$ of the entries and the evaluation (testing) of the learned tensor was done on another $10\%$ of the entries. The experimented is repeated $10$ times.  Figure \ref{fig:images}(a) shows the Ribeira image. 
\newline
\textbf{Tomato:} A tomato video sequence data set\footnote{\url{http://www.cs.rochester.edu/u/jliu/publications.html}} \cite{liu2013a,caiafa2014a} of size $242 \times 320 \times  167$. We follow the same experimental setup as for Ribeira data set  discussed above. \\
\textbf{Baboon:} An RGB image \footnote{\url{https://github.com/qbzhao/BCPF}} of a baboon used in \cite{Zhao2015} which is modeled as a $256\times 256\times 3 $ tensor. The experimental setup is same as that for Ribeira. Figure \ref{fig:images}(b) shows the image.\\
\textbf{Results.} Table~\ref{tbl:rmse_table} reports the root mean squared error (RMSE) on the test set that is averaged over ten splits. Our algorithms, TR-MM and TR-LS, obtain the best results, outperforming other trace norm based algorithms  on Ribeira, Tomato and Baboon data sets. Figure~\ref{fig:auc_fig}(a) shows the trade-off between the test RMSE and the training time of all the algorithms on Ribeira. It can be observed that both our algorithms converge to the lowest RMSE at a significantly faster rate compared to the other baselines. It is evident from the results that learning a mixture of non-sparse tensors, as learned by the proposed algorithms, helps in achieving better performance compared to the algorithms that learn a skewed sparse mixture of tensors.

\subsection{Link prediction}
In this application, the task is to predict missing or new links in knowledge graphs, social networks, etc. We consider the following data sets.\newline
\textbf{FB15k-237: } FB15k-237\footnote{\url{http://kristinatoutanova.com/}} \cite{toutanova2015a}  is a subset of FB15k dataset \cite{bordes2013a}, containing facts of the form subject-predicate-object (RDF) triples from Freebase knowledge graph. 
The subject and object noun phrases are called entities and the predicates are the relationships among them. For instance, in the RDF triple \{`Obama', `president of', `United States'\}, `Obama' and `United States' are entities, and `president of' is a relationship between them. The task is to predict relationships (from a given set of relations) between a pair of entities. 
FB15k-237  contains $14\, 951$ entities and $237$ relationships. It has  $310\,116$ observed relationships (links) between pairs of entities, which are {\it positive} samples. In addition, $516\, 606$ {\it negative} samples are generated following the procedure described in \cite{bordes2013a}. Negative samples are those \{entity$_a$, relationship$_c$, entity$_b$\} triples in which the  relationship$_c$ does not hold between entity$_a$ and entity$_b$. We model this task as a $14\,951 \times 14\,951 \times 237$ tensor completion problem. $\ten{Y}_{(a,b,c)}=1$ implies that relationship$_c$ exists between entity$_a$ and entity$_c$ (positive sample), and $\ten{Y}_{(a,b,c)}=0$ implies otherwise (negative sample). We keep $80\%$ of the observed entries for training  and the remaining $20\%$ for testing. \newline
\textbf{YouTube: } This is a link prediction data set\footnote{\url{http://leitang.net/data/youtube-data.tar}} \cite{tang2009a} having 5 types of interactions between $15\, 088$ users. The task is to predict the interaction (from a given set of interactions) between a pair of users. We model it as a  $15\, 088 \times 15\, 088 \times 5$ tensor completion problem. All the entries are known in this case. We randomly sample $0.8\%$ of the data  for training \cite{guo2017a} and another $0.8\%$ for testing. 

It should be noted that Hard, HaLRTC, and Latent do not scale to the full FB15k-237 and YouTube data sets as they need to store full tensor in memory. Hence, we follow \cite{guo2017a} to create a subset of the YouTube data set of size $1509 \times 1509 \times 5$ in which 1509 users with most number of links are chosen. We randomly sample $5\%$ of the data  for training and another $5\%$ for testing.

\textbf{Results.} All the above experiments are repeated on ten random train-test splits. Following  \cite{lu2011a,guo2017a}, the generalization performance for link prediction task is measured by computing the area under the ROC curve on the test set (test AUC) for each algorithm. Table~\ref{tbl:auc_table} report the average test AUC on YouTube (subset), Youtube (full) and  FB15k-237 data sets. The TR-MM algorithm achieves the best performance in all the link prediction tasks, while TR-LS achieves competitive performance.  This shows that the non-sparse mixture of tensors learned by TR-MM helps in achieving better performance. Figures~\ref{fig:auc_fig} (b) \& (c) plots the trade-off between the test AUC and the training time for FB15k-237 and YouTube, respectively. We can observe that TR-MM is the fastest to converge to a good AUC and take only a few iterations. 

We also conduct experiments to evaluate the performance of different algorithms in challenging scenarios when the amount of training data available is less. On the FB15k-237 data set, we vary the size of training data from $20\%$ to $80\%$ of the observed entries, and the remaining $20\%$ of the observed entries is kept as the test set. Figure~\ref{fig:auc_fig} (d) plots the performance of different algorithms on this experiment. We can observe that both our algorithms do significantly better than baselines in data scarce regimes.

\begin{table}
\centering
\caption{ Mean test AUC (higher is better) for link prediction problem. Our algorithms, TR-MM and TR-LS, obtain the best performance. `-' denotes the data set is too large for the algorithm to generate result. } \label{tbl:auc_table}
\setlength{\tabcolsep}{4pt}
\begin{tabular}{ llll } 
\toprule
  & \multicolumn{1}{c}{YouTube (subset)} & \multicolumn{1}{c}{YouTube (full)} & \multicolumn{1}{c}{FB15k-237}\\
\midrule
TR-MM & $\mathbf{0.957\pm0.001}$ & $\mathbf{0.932\pm0.001}$ & $\mathbf{0.823\pm0.001}$\\ 
TR-LS & $0.954\pm0.001$& $0.928\pm0.002$& $0.799\pm0.002$\\
FFW & $0.954\pm0.001$ &$0.929\pm0.001$ & $0.764\pm0.003$\\ 
Rprecon & $0.941\pm0.001$ & $0.926\pm0.001$ & $0.821\pm0.003$\\ 
geomCG & $0.941\pm0.001$ & $0.926\pm0.001$ & $0.785\pm0.014$\\ 
Hard & $0.954\pm0.001$ & -& -\\ 
Topt & $0.941\pm0.001$ & -& -\\ 
HaLRTC & $0.783\pm0.090$ & -& -\\ 
Latent & $0.945\pm0.001$ &-& -\\ 
T-svd & $0.941\pm0.001$& & \\
BayesCP &$0.949\pm0.002$ & -& - \\
\bottomrule
\end{tabular}
\end{table}





\subsection{Movie recommendation}
We also evaluate the algorithms on the MovieLens10M\footnote{\url{http://files.grouplens.org/datasets/movielens/ml-10m-README.html}} data set \cite{harper2015a}. This is a movie recommendation task --- predict the ratings given to movies by various users. MovieLens10M contains $10\,000\,054$ ratings of $10\,681$ movies given by $71\,567$ users. 
Following \cite{Kasai2016a}, we split the time into 7-days wide bins, forming a tensor of size $71\,567 \times 10\,681 \times 731$. For our experiments, we generate $10$ random train-test splits, where $80\%$ of the observed entries is kept for training and remaining $20\%$ is used for testing. Table~\ref{tbl:rmse_table} reports the average test RMSE on this task. 
It can be observed that the proposed  algorithms outperform state-of-the-art latent and scaled latent trace norm based algorithms.

\subsection{Results compared to other baseline algorithms}
\begin{table}[tb]
 \centering
 \caption{Ranks set for different datasets, at which the respective algorithms achieve best performance. For Tucker based algorithms, the rank is the Tucker rank which is different from the rank for the proposed algorithms.}\label{tbl:rank_tbl}
 \begin{tabular}{p{3.4cm}p{4.5cm}p{5cm}}
 \toprule
  Dataset & TR-MM/TR-LS  & Tucker-based algorithms\\
  \midrule
  Ribeira & (5,5,5)  & (15,15,6)\\
  Tomato & (10,10,10) &(15,15,15)\\
  Baboon & (4,4,3)&  (4,4,3)\\
  FB15k-237 &(20,20,1)  & (5,5,5)\\
  YouTube (subset) &(3,3,1)  & (5,5,5)\\
  YouTube (full) & (3,3,1) &(5,5,5) \\
  ML10M & (20,10,1) & (4,4,4)\\
  \bottomrule
 \end{tabular}
\end{table}

In addition to the trace norm based algorithms, we also compare against the following tensor decomposition algorithms:
\begin{enumerate}
\item Rprecon\footnote{\url{https://bamdevmishra.com/codes/tensorcompletion/}} \cite{Kasai2016a}: a Riemannian manifold preconditioning algorithm for the fixed multi-linear rank low-rank tensor completion problem.  
\item geomCG\footnote{\url{http://anchp.epfl.ch/geomCG}} \cite{Kressner2014a}: a fixed multi-linear rank low-rank tensor completion algorithm using Riemannian optimization on the manifold of fixed multi-linear rank tensors.  
\item Topt \cite{Filipovic2015a}: a non-linear conjugate gradient algorithm for Tucker decomposition. 
\item T-svd\footnote{\url{http://www.ece.tufts.edu/~shuchin/software.html}}\cite{Zhang2014}: a tubal rank based tensor completion algorithm.
\item BayesCP\footnote{\url{https://github.com/qbzhao/BCPF}}\cite{Zhao2015}:a Bayesian CP algorithm for tensor completion which has inbuilt rank tuning.
\end{enumerate}

As can be observed from Tables \ref{tbl:rmse_table} and \ref{tbl:auc_table}, the proposed TR-MM and TR-LS achieve better results than the baseline algorithms. However for the movie recommendation task, Rprecon achieves better results. It should be noted that Topt, T-svd, and BayesCP are not scalable for large scale datasets.

For all the datasets, we vary the rank between (3,3,3) and (25,25,25) for TR-MM and TR-LS algorithms. For Tucker based algorithms we vary the Tucker rank between (3,3,3) and (25,25,25). Table \ref{tbl:rank_tbl} shows the ranks at which the algorithms achieve best results with different datasets.

Algorithms based on the Tucker decomposition \cite{Filipovic2015a,Kasai2016a,Kressner2014a} model the tensor completion problem as a factorization problem. Tucker decomposition is a form of higher order PCA \cite{kolda2009a}, it decomposes a tensor into a core tensor multiplied by a matrix along each mode. Given the in complete tensor $\ten{Y}_{\Omega}$, the Tucker decomposition based tensor completion problem is:
\[
 \underset{\mat{A}_{k} \in \mathbb{R}^{n_k \times r_k}, k \in \{1,\dots,K\}}{\underset{\ten{G}\in \mathbb{R}^{r_1 \times \ldots \times r_K}}{\min}} \norm{\ten{Y}_{\Omega} - (\ten{G} \times_1 \mat{A}_1 \times_2 \ldots \times_K \mat{A}_K)_{\Omega} },
\]
where the multi-linear rank of the optimal solution is given by $(r_1, r_2, \ldots, r_K)$,  which is a user provided parameter in practical settings. CP based algorithms like BayesCP \cite{Zhao2015} are a special case of Tucker based modeling with $\ten{G}$ set to be  super diagonal and $r_1 = r_2 = \ldots = r_K = r$, which is a user provided parameter in practical setting. In particular, \cite{Zhao2015} discusses a probabilistic model. T-svd algorithm \cite{Zhang2014} models the tensor completion problem using the tensor singular value decomposition, and the rank here is the tensor tubal rank. Finally, algorithms based on latent norm regularization \cite{guo2017a,liu2013a,Signoretto2014a,tomioka2010a} model the completion problem by approximating the input tensor as a sum of individual tensors (\ref{eq_primal}). Each of these individual tensors is constrained to have a low-rank in one mode. The low-rank constraint is enforced on a different mode for each tensor. Due to the fundamental difference in modeling of these algorithms, the tensor ranks of the optimal solution obtained from algorithms that follow one approach cannot be compared with that of the algorithms that follow the other approach.

\subsection{Results on outlier robustness}
\begin{table}
\centering
\caption{Results on outlier robustness experiments. Our algorithm, TR-MM, is more robust to outliers than the competing baselines. The symbol `$-$' denotes the dataset is too large for the algorithm to generate result. }\label{tbl:robust}
\scalebox{0.81}
{
\setlength{\tabcolsep}{4pt} 
\begin{tabular}{l|l|ccccccccccc}
\toprule
 & \multicolumn{1}{c|}{$x$} & \multicolumn{1}{c}{TR-MM} & \multicolumn{1}{c}{TR-LS}&\multicolumn{1}{c}{FFW} & \multicolumn{1}{c}{Rprecon} & \multicolumn{1}{c}{geomCG} & \multicolumn{1}{c}{Hard} & \multicolumn{1}{c}{Topt} & \multicolumn{1}{c}{HaLRTC} & \multicolumn{1}{c}{Latent} & \multicolumn{1}{c}{T-svd} & \multicolumn{1}{c}{BayesCP}\\
\midrule
\multirow{2}{75pt}{Ribeira (RMSE)} 
  										& $0.05$ & $\mathbf{0.081}$  & $0.137$ &$0.095$ & $0.157$ & $0.258$ & $0.142$ & $0.169$ & $0.121$ & $0.103$ & $0.146$ & $0.201$ \\
  										& $0.10$ & $\mathbf{0.111}$ &$ 0.171$ &$0.112$ & $0.172$ & $0.373$ & $0.158$ & $0.188$ & $0.135$ & $0.120$ & $0.182$ & $0.204$ \\
\midrule
\multirow{2}{75pt}{FB15k-237 (AUC)} 
  											& $0.05$ & $\mathbf{0.803}$ & $0.768$&$0.734$ & $0.794$ & $0.764$ & $-$ & $-$ & $-$ & $-$ & $-$ & $-$ \\
  											& $0.10$ & $\mathbf{0.772}$ & $0.736$&$0.711$ & $0.765$ & $0.739$ & $-$ & $-$ & $-$ & $-$ & $-$ & $-$ \\
\bottomrule
\end{tabular}
}
\end{table}

We also evaluate TR-MM and the baselines considered in Table~\ref{tbl:bsline_table} for outlier robustness on hyperspectal image completion  and link prediction problems. In the Ribeira dataset, we add the standard Gaussian noise $(N(0,1))$ to randomly selected $x$ fraction of the  entries in the training set. The minimum and the maximum value of entries in the (original) Ribeira are $0.01$ and $2.09$, respectively. 
In FB15k-237 dataset, we flip randomly selected $x$ fraction of the entries in the training set, \textit{i.e.}, the link is removed if present and vice-versa. We experiment with $x=0.05$ and $x=0.10$. 

The results are reported in Table~\ref{tbl:robust}. We observe that our algorithm, TR-MM, obtains the best generalization performance and, hence, is the most robust to outliers. We also observe that trace norm regularized algorithms (TR-MM, FFW, Hard, HaLTRC, Latent) are relatively more robust to outliers than Tucker-decomposition based algorithms (Rprecon, geomCG, Topt), CP-decomposition based algorithm (BayesCP), and tensor tubal-rank based algorithm (T-svd).

\section{Conclusion}
\label{sec:conclusion}

In this paper, we introduce a variant of trace norm regularizer for low-rank tensor completion problem which learns the tensor as a non-sparse combination of tensors. The number of tensors in this combination is equal to the number of modes ($K$). Existing works \cite{tomioka2010a,tomioka2013a,wimalawarne2014a,guo2017a} learn a sparse combination of tensors, essentially learning the tensor as a low-rank matrix and losing higher order information in the available data. Hence, we recommend learning a non-sparse combination of tensors in trace norm regularized setting, especially since $K$ is typically a small integer in most real-world applications. 
In our experiments, we observe better generalization performance with the proposed regularization. Theoretically, we provide the following result on the reconstruction error in the context of recovering an unknown tensor $\ten{W}^{*}$  from noisy observation. 
\begin{lemma}\label{gen_theorem}
Let $\ten{W}^*$ be the true tensor to be recovered from observed $\ten{Y}$, which is obtained as $\ten{Y} = \ten{W}^*+\ten{E}$, where $\ten{E} \in \mathbb{R}^{n_1\times \ldots \times n_K}$ is the noise tensor. Assume that the regularization constant 
 $\lambda$ satisfies $\lambda \leq 1/ (\sum_{k=1}^{K} \|\ten{E}_k\|_{\infty}^{2})^{1/2}$  then the estimator $$\hat{\ten{W}} = \underset{\ten{W}}{argmin} (\frac{1}{2} \|\ten{Y} - \ten{W}\|_F^2 + \frac{1}{\lambda} \sum\limits_{k} \|\ten{W}^{(k)}_{k}\|_*^2),$$ satisfies the inequality  
 $ \|\hat{\ten{W}} - \ten{W}^*\|_F \leq \frac{2}{\lambda}\sqrt{\underset{k}{\min}~ n_k}$.
When noise approaches zero, \textit{i.e.}, $\ten{E}\rightarrow 0$, the right hand side also approaches zero.
\end{lemma}

The proof of the above lemma is in Section \ref{supp:sec_reconstruction_proof}. A similar result on the \textit{latent trace norm}, which learns a sparse combination of tensors, is presented in \cite{tomioka2013a}.

We present a dual framework to analyze the proposed tensor completion formulation. 
This leads to  novel fixed-rank formulations, for which we exploit the Riemannian framework to develop scalable trust region algorithms. 
In experiments, our algorithm TR-MM obtains better generalization performance and is more robust to outliers than state-of-the-art low-rank tensor completion algorithms. 
Overall both TR-MM and TR-LS algorithms achieve better performance on various completions tasks.

\section*{Acknowledgement}
Most of this work was done when Madhav Nimishakavi (as an intern), Pratik Jawanpuria, and Bamdev Mishra were at Amazon.com.  

\bibliographystyle{amsalpha}
\bibliography{arXiv_tensor_completion}




\appendix

\clearpage

\section{Proof of Lemma \ref{gen_theorem}}\label{supp:sec_reconstruction_proof}
\begin{proof}
Let $R(\ten{W}) = (\sum\limits_{k} \|\ten{W}^{(k)}_{k}\|_*^2)^{\frac{1}{2}}$,
from the optimality of $\hat{\ten{W}}$, we have 
\begin{equation}
\label{supp:eq_w_hat}
\ten{Y} - \hat{\ten{W}} \in \frac{1}{\lambda} \partial R(\hat{\ten{W}})^2 ,
\end{equation}

where $\partial R(\hat{\ten{W}})^2 $ is the subdifferential of $R(\ten{W})^2$ at $\ten{W} = \hat{\ten{W}}$. \\
From the triangle inequality, we obtain
\begin{equation*}
 \norm{\hat{\ten{W}} -\ten{W}^* }_F \leq \norm{\hat{\ten{W}} - \ten{Y}}_F + \norm{\ten{E}}_F . 
\end{equation*}

First term in the right hand side of the above inequality satisfies
\begin{equation}
\label{supp:eq_ineq_w_y}
\norm{\hat{\ten{W}} - \ten{Y}}_F \leq \sqrt{\underset{k}{\min}  ~ n_k } \norm{\hat{\ten{W}} - \ten{Y}}_{\underline{S_{\infty/2}}}, 
\end{equation}
this can be seen from the following inequalities. For any tensor $\ten{W}$, 
\[
 \norm{\ten{W}}_F^2 \leq \underset{k}{\min} ~ n_k \norm{\ten{W}_k}_{S_{\infty}}^2 \leq \underset{k}{\min} ~ n_k \sum_{k=1}^{K} \norm{\ten{W}_k}_{S_{\infty}}^2.
\]

For any $\ten{X} \in \partial R(\ten{W})^2$, we have $R_*(\ten{X})^2 \leq 1 $, where $R_*(\ten{X})^2 \coloneqq \norm{\ten{X}}_{\underline{S_{\infty/2}}}^2$ is the dual of $R(\ten{X})^2$. Using this in \eqref{supp:eq_ineq_w_y} , 
\[
 \norm{\hat{\ten{W}} -\ten{Y} }_F \leq \frac{1}{\lambda} \sqrt{\underset{k}{\min}  ~ n_k }.
\]

Following similar analysis and from the assumption we have 

\[
 \norm{\ten{E}}_F \leq \sqrt{\underset{k}{\min}  ~ n_k } \norm{\ten{E}}_{\underline{S_{\infty/2}}} \leq \frac{1}{\lambda} \sqrt{\underset{k}{\min}  ~ n_k }, 
\]

\end{proof}

\section{Optimization on spectrahedron}\label{supp:spec_opt}
\renewcommand\theequation{A\arabic{equation}}
\setcounter{equation}{0}

We are interested in the optimization problem of the form
\begin{equation}\label{sup:eq:problem_formulation}
\begin{array}{lll}
\minop_{\Theta \in \P^d} & f(\Theta), \\
\end{array}
\end{equation}
where $\P^d$ is the set of $d\times d$ positive semi-definite matrices with unit trace and $f: \P^d \rightarrow \R $ is a smooth function. A specific interest is when we seek matrices of rank $r$. Using the parameterization $\Theta = \bU \bU^\top$, the problem (\ref{sup:eq:problem_formulation}) is formulated as 
\begin{equation}\label{sup:eq:fixed_rank_problem_formulation}
\begin{array}{lll}
\minop_{\bU \in \mathcal{S}_r^d} & f(\bU \bU^\top),
\end{array}
\end{equation}
where $\mathcal{S}_r^d\coloneqq \{\bU \in \mathbb{R}^{d \times r}: \| \bU\|_F = 1 \}$, which is called the spectrahedron manifold \cite{journe2010a}. It should be emphasized the objective function in (\ref{sup:eq:fixed_rank_problem_formulation}) is \emph{invariant} to the post multiplication of $\bU$ with orthogonal matrices of size $r\times r$, i.e., $\bU \bU^\top =\bU \bQ (\bU \bQ)^\top$ for all $\bQ \in \OG{r}$, which is the set of orthogonal matrices of size $r\times r$ such that $\bQ \bQ^\top = \bQ^\top \bQ = \bI$. An implication of the this observation is that the minimizers of (\ref{sup:eq:fixed_rank_problem_formulation}) are no longer isolated in the matrix space, but are isolated in the quotient space, which is the set of equivalence classes $[\bU] \coloneqq \{ \bU \bQ: \bQ \bQ^\top = \bQ^\top \bQ = \bI \}$. Consequently, the search space is 
\begin{equation}\label{sup:eq:equivalence_manifold}
\begin{array}{lll}
\mathcal{M}  \coloneqq  \mathcal{S}_r^d/\OG{r}.
\end{array}
\end{equation}
In other words, the optimization problem (\ref{sup:eq:fixed_rank_problem_formulation}) has the structure of optimization on the \emph{quotient} manifold, i.e.,
\begin{equation}\label{sup:eq:manifold_formulation}
\begin{array}{lll}
\minop_{[\bU] \in \mathcal{M}} & f([\bU]),
\end{array}
\end{equation}
but numerically, by necessity, algorithms are implemented in the matrix space $\mathcal{S}_r^d$, which is also called the \emph{total space}.

Below, we briefly discuss the manifold ingredients and their matrix characterizations for (\ref{sup:eq:manifold_formulation}). Specific details of the spectrahedron manifold are discussed in \cite{journe2010a}. A general introduction to manifold optimization and numerical algorithms on manifolds are discussed in \cite{absil08a}.

\subsection{Tangent vector representation as horizontal lifts}
Since the manifold $\mathcal{M}$, defined in (\ref{sup:eq:equivalence_manifold}), is an abstract space, the elements of its tangent space $T_{[\bU]} \mathcal{M}$ at $[\bU]$ also call for a matrix representation in \changeBM{the tangent space $T_\bU \mathcal{S}_r^d$} that respects the equivalence \changeBM{relation} $\bU \bU^\top =\bU \bQ (\bU \bQ)^\top$ for all $\bQ \in \OG{r}$. Equivalently, \changeBM{the} matrix representation of $T_{[\bU]} \mathcal{M}$ should be restricted to the directions in the tangent space $T_{\bU}   \mathcal{S}_r^d$ on the total space $  \mathcal{S}_r^d$ at ${ \bU}$ that do not induce a displacement along the equivalence class $[\bU]$. In particular, we decompose  $T_\bU \mathcal{S}_r^d$ into complementary subspaces,  the \emph{vertical} $\mathcal{V}_{\bU}$ and \emph{horizontal} $\mathcal{H}_\bU$ subspaces, such that $ \mathcal{V}_{  \bU}  \oplus \mathcal{H}_{  \bU} = T_{  \bU}  \mathcal{S}_r^d$. 

\begin{table*}[t]
\begin{center}
\caption{Matrix characterization of notions on the quotient manifold $\mathcal{S}_r^d/\OG{r}$.}
\label{sup:tab:spaces} 
\begin{tabular}{ p{10cm} | p{5cm}} 
\toprule
Matrix representation of an element & $\bU $    \\ 
&  \\
Total space $\mathcal{S}_r^d$ & $\{\bU \in \R^{d\times r}: \|\bU\|_F = 1 \}$   \\ 
 &  \\
Group action &   $\bU \mapsto \bU \bQ$, where $\bQ \in \OG{r}$.  \\  
 & \\
Quotient space ${\mathcal M}$ & $\mathcal{S}_r^d /\OG{r}$  \\ 
& \\
Tangent vectors in the total space $\mathcal{S}_r^d $ at $\bU$ &$ \{ \bZ \in \R^{d \times r} :  \trace(\bZ^\top \bU) = 0 \}$  \\
&  \\
Metric between the tangent vector $\xi_\bU, \eta_\bU \in T_\bU \mathcal{S}_r^d$
& $\trace(\xi_\bU ^\top \eta_\bU)$ \\
& \\
Vertical tangent vectors at $\bU$ &  $\{\bU {\bf \Lambda}: {\bf \Lambda}\in \R^{r\times r},{\bf \Lambda}^\top = - {\bf\Lambda} \}$ \\
&  \\
Horizontal tangent vectors & $\{ \xi_\bU \in T_{\bU} \mathcal{S}_r^d : \xi_\bU^\top \bU = \bU^\top \xi_\bU\} $\\
&  \\
\bottomrule
\end{tabular}
\end{center} 
\end{table*}

The vertical space $\mathcal{V}_{ \bU}$ is the tangent space of the equivalence class $[\bU]$. On the other hand, the horizontal space $\mathcal{H}_{  \bU}$, which is \changeBM{any complementary subspace} to $\mathcal{V}_{  \bU}$ in $T_\bU  \mathcal{S}_r^d$, provides a valid matrix representation of the abstract tangent space $T_{[\bU]} \mathcal{M}$. An abstract tangent vector $\xi_{[\bU]} \in T_{[\bU]} \mathcal{M}$ at $ [\bU]$ has a unique element in the horizontal space $ {\xi}_{ {\bU}}\in\mathcal{H}_{ {\bU}}$ that is called its \emph{horizontal lift}. \changeBM{Our specific choice of the horizontal space is the subspace of $T_\bU  \mathcal{S}_r^d$ that is the \emph{orthogonal complement} of $\mathcal{V}_{  x}$ in the sense of a \emph{Riemannian metric}}.

The Riemannian metric at a point on the manifold is a inner product that is defined in the tangent space. An additional requirement is that the inner product needs to be \emph{invariant} along the equivalence classes \cite[Chapter~3]{absil08a}. One particular choice of the Riemannian metric on the total space $\mathcal{S}_r^d$ is
\begin{equation}\label{sup:eq:metric}
	 \langle {\xi}_{ {\bU}}, {\eta}_{ {\bU}} \rangle_ \bU:= \trace(\xi_\bU ^\top \eta_\bU),
\end{equation}
where $\xi_\bU, \eta_\bU \in T_\bU \mathcal{S}_r^d$. The choice of the metric (\ref{sup:eq:metric}) leads to a natural choice of the metric on the quotient manifold, i.e., 
\begin{equation*}\label{sup:eq:metric_quotient}
	 \langle {\xi}_{ {[\bU]}}, {\eta}_{ [{\bU}]} \rangle_ {[\bU]}:= \trace(\xi_\bU ^\top \eta_\bU),
\end{equation*}
where $\xi_{[\bU]}$ and $ \eta_{[\bU]}$ are abstract tangent vectors in $T_{[\bU]} \mathcal{M} $ and $\xi_\bU$ and $\eta_\bU$ are their horizontal lifts in the total space $\mathcal{S}_r^d$, respectively. Endowed with this Riemannian metric, the quotient manifold $\mathcal{M}$ is called a \emph{Riemannian} quotient manifold of $\mathcal{S}_r^d$.

Table \ref{sup:tab:spaces} summarizes the concrete matrix operations involved in computing horizontal vectors. 

Additionally, starting from an arbitrary matrix (an element in the ambient dimension $\R^{d\times r}$), two linear projections are needed: the first projection $\Psi_\bU$ is onto the tangent space $T_\bU \mathcal{S}_r^d$ of the total space, while the second projection $\Pi_\bU$ is onto the horizontal subspace $\mathcal{H}_\bU$. 

Given a matrix $\bZ \in \R^{d\times r}$, the projection operator $\Psi_\bU: \R^{d\times r} \rightarrow T_\bU \mathcal{S}_r^d : \bZ \mapsto \Psi_\bU(\bZ)$ on the tangent space is defined as 
\begin{equation}\label{sup:eq:tangent_space_projector}
\begin{array}{lll}
\Psi_\bU(\bZ) = \bZ - \trace(\bZ^\top  \bU) \bU. 
\end{array}
\end{equation}

Given a tangent vector $\xi_\bU \in T_\bU \mathcal{S}_r^d$, the projection operator $\Pi_\bU: T_\bU \mathcal{S}_r^d \rightarrow \mathcal{H}_\bU: \xi_\bU \mapsto \Pi_\bU(\xi_\bU)$ on the horizontal space is defined as 
\begin{equation}\label{sup:eq:horizontal_space_projector}
\begin{array}{lll}
\Pi_\bU(\xi_\bU) = \xi_\bU -  \bU {\bf \Lambda},
\end{array}
\end{equation}
where $\bf \Lambda$ is the solution to the \emph{Lyapunov} equation
\begin{equation*}
(\bU^\top \bU){\bf \Lambda} + {\bf \Lambda} (\bU^\top \bU) =  \bU^\top \xi_\bU- \xi_\bU^\top \bU.
\end{equation*}

\subsection{Retractions from horizontal space to manifold}

An iterative optimization algorithm involves computing a search direction ({\it e.g.,} the gradient direction) and then moving in that direction. The default option on a Riemannian manifold is to move along geodesics, leading to the definition of the exponential map. Because the calculation of the exponential map can be computationally demanding, it is customary in the context of manifold optimization to relax the constraint of moving along geodesics. The exponential map is then relaxed to a \emph{retraction} operation, which is any map $R_\bU : \mathcal{H}_
\bU \rightarrow \mathcal{S}_r^d: \xi_\bU \mapsto R_\bU(\xi_\bU) $ that locally approximates the exponential map on the manifold \cite[Definition~4.1.1]{absil08a}. On the spectrahedron manifold, a natural retraction of choice is 
\begin{equation*}\label{sup:eq:retraction}
R_\bU(\xi_\bU) := (\bU + \xi_\bU)/\|\bU + \xi_\bU \|_F,
\end{equation*} 
where $\| \cdot \|_F$ is the Frobenius norm and $\xi_\bU$ is a search direction on the horizontal space $\mathcal{H}_\bU$.

An update on the spectrahedron manifold is, thus, based on the update formula
$\bU_+ = R_\bU (\xi_\bU) $.

\subsection{Riemannian gradient and Hessian computations}
The choice of the invariant metric (\ref{sup:eq:metric}) and the horizontal space turns the quotient manifold $\mathcal{M}$ into a \emph{Riemannian submersion} of $(\mathcal{S}_r^d, \langle \cdot, \cdot \rangle)$. As shown by \cite{absil08a}, this special construction allows for a convenient matrix characterization of the gradient and the Hessian of a function on the abstract manifold $\mathcal{M}$.

The matrix characterization of the Riemannian gradient is 
\begin{equation}\label{sup:eq:rgrad}
\begin{array}{lll}
\grad_\bU f = \Psi_\bU(\nabla_\bU f),
\end{array}
\end{equation}
where $\nabla_\bU f$ is the Euclidean gradient of the objective function $f$ and $\Psi_\bU$ is the tangent space projector (\ref{sup:eq:tangent_space_projector}).

An iterative algorithm that exploits second-order information usually requires the Hessian applied along a search direction. This is captured by the Riemannian Hessian operator $\hess$, whose matrix characterization, given a search direction $\xi_\bU \in \mathcal{H}_\bU$, is 
\begin{equation}\label{sup:eq:rhess}
\begin{array}{lll}
\hess_{\bU}[\xi_{\bU}] = \Pi_\bU\Big({\rm D}\nabla f[\xi_\bU] -  \trace((\nabla_\bU f)^\top \bU)\xi_\bU 
 - \trace((\nabla_\bU f)^\top \xi_\bU + ({\rm D}\nabla f[\xi_\bU])^\top \bU)\bU\Big), \\ 
\end{array}
\end{equation}
where ${\rm D}\nabla f[\xi_\bU]$ is the directional derivative of the Euclidean gradient $\nabla_\bU f$ along $\xi_\bu$ and $\Pi_\bU$ is the horizontal space projector (\ref{sup:eq:horizontal_space_projector}).

Finally, the formulas in (\ref{sup:eq:rgrad}) and (\ref{sup:eq:rhess}) that the Riemannian gradient and Hessian operations require only the expressions of the standard (Euclidean) gradient of the objective function $f$ and the directional derivative of this gradient (along a given search direction) to be supplied.

\end{document}